\documentclass[11pt]{amsart}

\usepackage[english]{babel}
\usepackage[utf8x]{inputenc}
\usepackage[T1]{fontenc}

\usepackage[a4paper,top=3cm,bottom=2cm,left=3cm,right=3cm,marginparwidth=1.75cm]{geometry}

\usepackage{amsmath, amssymb, epsfig}
\usepackage{mathrsfs}
\usepackage{color}
\usepackage[obeyspaces]{url}
\usepackage{algorithmicx}
\usepackage[ruled]{algorithm}
\usepackage{algpseudocode}
\usepackage{graphicx}
\usepackage{caption}
\usepackage{subcaption}
\usepackage{float}
\usepackage{amsthm}
\usepackage{multirow}
\usepackage{booktabs}
\usepackage{soul}
\usepackage{longtable}
\usepackage{wrapfig}
\usepackage{comment}
\usepackage{environ}

\usepackage{pgf,tikz,pgfplots}
\pgfplotsset{compat=1.15}
\usepackage{mathrsfs}
\usetikzlibrary{arrows}
\usepackage{amssymb,fancyhdr,txfonts,pxfonts}

\usepackage{fourier} 

\usepackage{bm} 
\usepackage{fullpage}

\usepackage{hyperref}

\theoremstyle{plain}
\newtheorem{theorem}{Theorem}[section]

\newtheorem{lemma}[theorem]{Lemma}

\theoremstyle{definition}
\newtheorem{definition}[theorem]{Definition}

\newtheorem*{remark}{Remark}

\numberwithin{theorem}{section}
\numberwithin{equation}{section}

\makeatletter
\newcommand{\vast}{\bBigg@{3}}
\newcommand{\Vast}{\bBigg@{5}}
\makeatother

\DeclareMathOperator*{\supp}{supp}
\DeclareMathOperator*{\rank}{rank}

\DeclareMathOperator*{\argmin}{arg min}

\newcommand{\vct}[1]{\bm{#1}}

\def \x {\vct{x}}

\def \rank {{\rm rank }}

\newcommand{\vd}{{\mathbf{d}}}

\newcommand{\vs}{{\mathbf{s}}}

\newcommand{\vv}{{\mathbf{v}}}
\newcommand{\vw}{{\mathbf{w}}}
\newcommand{\vx}{{\mathbf{x}}}

\newcommand{\vone}{{\mathbf{1}}}

\newcommand{\vA}{{\mathbf{A}}}
\newcommand{\vB}{{\mathbf{B}}}

\newcommand{\vE}{{\mathbf{E}}}
\newcommand{\vF}{{\mathbf{F}}}
\newcommand{\vG}{{\mathbf{G}}}

\newcommand{\vI}{{\mathbf{I}}}

\newcommand{\vS}{{\mathbf{S}}}

\newcommand{\vU}{{\mathbf{U}}}

\newcommand{\vW}{{\mathbf{W}}}
\newcommand{\vX}{{\mathbf{X}}}
\newcommand{\vY}{{\mathbf{Y}}}
\newcommand{\vZ}{{\mathbf{Z}}}

\newcommand{\vPhi}{{\mathbf{\Phi}}}

\newcommand{\cL}{{\mathcal{L}}}

\newcommand{\loss}{L}

\algtext*{EndProcedure}
\algtext*{EndFor}
\algtext*{EndWhile}

\usepackage[english]{babel}
\usepackage[utf8x]{inputenc}
\usepackage[T1]{fontenc}

\usepackage[a4paper,top=3cm,bottom=2cm,left=3cm,right=3cm,marginparwidth=1.75cm]{geometry}

\makeatletter
\newsavebox{\measure@tikzpicture}
\NewEnviron{scaletikzpicturetowidth}[1]{%
  \def\tikz@width{#1}%
  \def\tikzscale{1}\begin{lrbox}{\measure@tikzpicture}%
  \BODY
  \end{lrbox}%
  \pgfmathparse{#1/\wd\measure@tikzpicture}%
  \edef\tikzscale{\pgfmathresult}%
  \BODY
}
\makeatother

\begin{document}

\title{Neural Nonnegative Matrix Factorization for Hierarchical Multilayer Topic Modeling}

\thanks{DN, JH, ES, JV, and DM are grateful to and were partially supported by NSF CAREER DMS \#1348721 and NSF BIGDATA \#1740325.  This work is based upon work completed at the UCLA CAM REU during Summer 2018 which was funded by NSF DMS \#1659676.  JH is additionally grateful to and was partially supported by NSF DMS \#2211318.
}

\author{Tyler Will}

\author{Runyu Zhang}

\author{Eli Sadovnik}

\author{Mengdi Gao}

\author{Joshua Vendrow}

\author{Jamie Haddock}

\author{Denali Molitor}

\author{Deanna Needell}

\begin{abstract}
We introduce a new method based on nonnegative matrix factorization, Neural NMF, for detecting latent hierarchical structure in data. Datasets with hierarchical structure arise in a wide variety of fields, such as document classification, image processing, and bioinformatics. Neural NMF recursively applies NMF in layers to discover overarching topics encompassing the lower-level features.
We derive a backpropagation optimization scheme that allows us to frame hierarchical NMF as a neural network. We test Neural NMF on a synthetic hierarchical dataset, the 20 Newsgroups dataset, and the MyLymeData symptoms dataset.
Numerical results demonstrate that Neural NMF outperforms other hierarchical NMF methods on these data sets and offers better learned hierarchical structure and interpretability of topics.
\end{abstract}

\maketitle

\section{Introduction}

As the size of available data continues to grow, scalable approaches for extracting meaningful latent trends within large-scale data, or reducing redundant information within the data have become active and important focuses of research.
Within this flourishing area of research, topic modeling approaches have received particular interest; topic modeling is a popular class of machine learning techniques that cluster and classify observations to reveal latent themes in a dataset.  Algorithms for topic modeling often find application in the domain of document classification and clustering~\cite{BNJ03,gaussier2005relation,shahnaz2006document,xu2003document,berry2005email,pauca2004text}, but more recently have found use in
applications
such as
image processing
~\cite{guillamet2002non,hoyer2002non,LSE99,FP05},
financial data mining~\cite{de2008analysis},
audio processing \cite{cichocki2006new,gemmeke2013exemplar}, genetics \cite{liu2017regularized}, and bioinformatics~\cite{LTD16}.
An approach related to, but distinct from, topic modeling is feature extraction. While topic modeling seeks to abstract the dataset and represent data points by these topics, feature extraction aims to find a few key, representative features that best represent the data set for the task at hand (e.g., classification) \cite{GEG06}.  Many of the most popular models for these tasks are built upon, or are simply, dimension-reduction techniques which aim to reduce the dimension of the representation of the data; see e.g.,~\cite{bishop2006pattern}.

Nonnegative matrix factorization (NMF) is a popular method in machine learning because it is able to both perform feature extraction and generate topic models \cite{BUC08,KCP15}.
Users of NMF specify
the number of topics believed to be in the dataset; the model then identifies
representative dictionary elements (topics) and coefficients which represent each element of the data set in terms of the topics (thereby lowering the representative dimension of the data).
Users often explore topics at different resolutions (number of topics) and seek relationships between the topics learned at various levels of granularity.
NMF, however, does not inherently discover the relationships between topics learned at these different levels. A popular alternative, known as \emph{hierarchical NMF (HNMF)}, is to sequentially apply NMF to learn the relationship between NMF topics at different levels of granularity.

 Borrowing techniques from neural networks, we seek to modify HNMF to

 illustrate the relationships between topics learned at differing levels of granularity, and to specifically provide a hierarchical representation of how topics at finer granularity relate to topics at courser granularity, while avoiding the often high approximation error of naive application of HNMF. An advantage of methods which illustrate hierarchical relationships among the topics
 over classical NMF is that practitioners can examine the results for multiple numbers of topics without running NMF multiple times.  Additionally, while the recovery error increases as the number of layers in a hierarchical model increases, these models have the desirable property that they immediately illustrate the hierarchical structure of the latent topics.   In applications in which the learned sub-topics represent known data clusters, the learned hierarchy could provide unknown cluster-level relationships within the data.

 We propose a new hierarchical NMF method, which we denote \emph{Neural NMF}, and illustrate its promise on a synthetic hierarchical data set, the 20 Newsgroups dataset, and MyLymeData, a real dataset containing survey data of Lyme disease patients collected by Lymedisease.org.\footnote{An extended abstract of this work appeared in \emph{Proc. Interational Workshop on Computational Advances in Multi-Sensor Adaptive Processing (CAMSAP) 2019}~\cite{GHMNSWZ19}.} On the synthtic data set, Neural NMF outerperforms other hierarchical NMF methods in both reconstruction error and classification accuracy. On the 20 Newsgroups dataset, Neural NMF outperforms other methods in classification accuracy and offers a better hierarchical structure and interpretability of the topics. We also illustrate the ability of Neural NMF to identify meaningful, and even surprising, hierarchical topic structure on real world data with the MyLymeData set.  These initial results indicate that the Neural NMF method can vastly improve both the reconstruction of the overall hierarchical NMF model as well as improving the learned inter-layer hierarchical structure and interpretability of topics at each layer.

 \subsection{Organization}
The remainder of our paper is organized as follows. In the remainder of the introduction, we briefly introduce notations and conventions in Section~\ref{subsec:notation}; introduce the foundational models which inspired our own: NMF in Section~\ref{sec:intro to nmf}, semi-supervised NMF in Section~\ref{subsec:SSNMF}, hierarchical NMF in Section~\ref{sec:HNMF}, and Deep NMF in Section~\ref{subsec:DNMF}; and finally review briefly some further related work in Section~\ref{subsec:relatedwork}.  In Section~\ref{sec:proposedmethod}, we introduce our approach which consists of a forward-propogation process
and a backward-propogation process (detailed in Section~\ref{subsec:backwardpropogation}; here we also include the statements of our main theoretical results that derive the necessary gradient information for backpropagation). In Section~\ref{sec:experimentalresults}, we empirically test Neural NMF on a small, synthetic dataset, the 20 Newsgroups dataset, and the MyLymeData dataset. Finally, in Section~\ref{sec:conclusion} we offer some conclusions and future directions.  The proofs of our main theoretical results are left to Appendix~\ref{sec:appendix} so as to not distract from the main focus of the paper: the promise of Neural NMF in applications.

\subsection{Notation}\label{subsec:notation}
We distinguish matrices and vectors from scalar quantities using bold font. For a matrix $\vF$, the notations $\vF_{i, :}$ and $\vF_{:, j}$ denote row $i$ and column $j$, respectively.  For sets of indices $T$ and $S$, we take $\vF_{T,:}$ and $\vF_{:,S}$ to mean the matrix obtained by restricting to the rows of $\vF$ in $T$ or the columns of $\vF$ in $S$, respectively.  By extension, $\vv_T$ is the vector $\vv$ restricted to the entries with indices in T. We denote the Moore-Penrose pseudoinverse of $\vF$ as $\vF^\dagger$.  Entrywise (Hadamard) multiplication and division between $\vF$ and $\vG$ are denoted by $\vF \odot \vG$ and $\frac{\vF}{\vG}$, respectively. The vector of length $k$ with all entries equal to one is denoted $\vone_k$, while $\vI$ indicates the identity matrix of compatible dimensions in all circumstances.  We perform subscript (indicial) operations before superscript (e.g., pseudoinversion, transposition) operations whenever applicable. The set $[0, \infty)^k$ is denoted $\mathbb{R}_+^{k}$.  In methods with $\cL$ layers, we use $\vF^{(\ell)}$ to mean the value of matrix $\vF$ at layer $\ell$.  We similarly use $k^{(\ell)}$ as the number of topics at layer $\ell$.  In any supervised setting, we use $P$ as the number of total classes. Finally, we denote the set of integers $\{1, 2, \dots, m\}$ as $[m]$. We let $\|\cdot\|$ denote the Frobenius norm ($\ell_2$ norm in the case of vector input) throughout, unless otherwise noted.

\subsection{NMF}
\label{sec:intro to nmf}
For a given data matrix $\vX \in \mathbb{R}_+^{N \times M} $ and model rank $k$, nonnegative matrix factorization (NMF) seeks  $\vA \in \mathbb{R}_+^{N \times k}$ and $\vS \in \mathbb{R}_+^{k \times M}$ such that $\vX \approx \vA\vS$.  To find $\vA$ and $\vS$, we wish to solve the optimization problem
\begin{equation}
\label{eq:nmf objective function}
\min_{ \vA \geq 0, \vS \geq 0} ||\vX - \vA\vS||^2.
\end{equation}
The nonnegativity restriction on $\vX$, $\vA$, and $\vS$ differentiates NMF from other topic modeling and feature extraction techniques, including principal component analysis (PCA) and autoencoding. Since all values are nonnegative, only additive combinations are allowed to approximate every data point in $\vX$.  This allows for a natural and intuitive `parts-based' representation \cite{LSE99}; there are no topics that contribute in a negative manner.
We can view the columns of the $\vA$ matrix as vectors of important features of $\vX$, or from a topic modeling perspective, as $k$ hidden themes in our data. In this view, the $\vS$ matrix provides the coefficients to represent each data point of $\vX$ as a linear combination of the hidden themes (the columns of $\vA$).  Moreover, we can view the product $\vA\vS$ as a low-rank approximation of $\vX$, since

\begin{equation}
\label{eq:rank bound AS}
\rank (\vA\vS) \leq \min (\rank(\vA), \rank(\vS)) \leq k.
\end{equation}

\par The problem \eqref{eq:nmf objective function} is convex in $\vA$ and $\vS$ separately (holding the other fixed), but nonconvex in both.  We therefore cannot expect to consistently find a global minimum as a solution.  Several techniques exist for finding local minima; most schemes alternate fixing $\vA$ and $\vS$ while updating the other variable in order to iteratively decrease \eqref{eq:nmf objective function}.  One of the most popular of these methods is the multiplicative updates methods \cite{lee2001algorithms}.
Other algorithms for approximating the solution to \eqref{eq:nmf objective function} include the alternating least squares and rank-one residue iteration (RRI) methods \cite{N08}.  Although these offer potential advantages over a multiplicative updates scheme, the latter can be readily extended to deal with more complex cost functions, such as that of semisupervised NMF, which is discussed in the next section.

\subsection{Semisupervised NMF (SSNMF)}\label{subsec:SSNMF}
A natural extension of NMF is to take advantage of any known label information in the factorization \cite{LYC10}.  Suppose $\vY \in \mathbb{R}^{P \times M}$ is a matrix containing label information for $M$ objects in $P$ classes.  Let $\vW \in \mathbb{R}^{N \times M}$ be the binary indicator matrix for the data, that is \begin{equation*}
\vW_{n,m} = \begin{cases}
1 &  \vX_{n,m} \text{ is known} \\
0 & \text{otherwise,}
\end{cases}
\end{equation*} and let $\vZ \in \mathbb{R}^{P \times M}$ be the binary indicator matrix for the labels, with
\[ \vZ_{:, m} =
\begin{cases}
\vone_P & \text{label for object } m \text{ is known} \\
0 & \text{otherwise}.
\end{cases}
\]
We can incorporate label information and manage missing data by adjusting our problem to
\begin{equation}
\label{eq:ssnmf objective function}
\min_{\vA \geq 0, \vS \geq 0, \vB \geq 0} \underbrace{\|\vW \odot(\vX-\vA\vS)\|_F^2}_\text{Reconstruction Error} + \lambda \underbrace{\|\vZ \odot(\vY-\vB\vS) \|_F^2}_\text{Classification Error}.
\end{equation} \\
The resulting $\vB \in \mathbb{R}^{P \times k}$ is a linear classification matrix, with $\vB_{i, :}$
indicating the association between class $i$ and each of the learned topics.
The error \eqref{eq:ssnmf objective function} balances the reconstruction error on the data that are known with the classification error on the labels that are known.  The relative importance of the classification error is controlled by the user-defined hyperparameter $\lambda$. The multiplicative update minimization strategy discussed in Section \ref{sec:intro to nmf} can be extended to handle this setting \cite{LYC10}.  Generalizations of this SSNMF model have been recently proposed in~\cite{AGHKKLLMMNSW20}.

\subsection{Hierarchical NMF (HNMF)}
\label{sec:HNMF}
\par A further extension of both NMF and semisupervised NMF seeks to illuminate hierarchical structure by recursively factorizing the $\vS$ matrices.  By performing NMF with $k = k^{(0)}$, we reveal $k^{(0)}$ hidden topics in the data; by repeating the factorization on the $\vS$ matrix with $k = k^{(1)}$, we further collect the data into $k^{(1)}$ supertopics among the $k^{(0)}$ topics.  This process for $\cL$ layers approximately factors the data matrix as
\begin{align}
\begin{split}
\vX &\approx \vA^{(0)}\vS^{(0)}, \\
\vX &\approx \vA^{(0)}\vA^{(1)}\vS^{(1)}, \\
& \vdots \\
\vX &\approx \vA^{(0)}\vA^{(1)}\cdots \vA^{(\cL)}\vS^{(\cL)}.
\end{split}
\end{align}
The $\vA^{(i)}$ matrix represents the how the subtopics at layer $i$ collect into the supertopics at layer $i+1$.

Note that as $\cL$ increases, the error
\begin{equation}
\|\vX - \vA^{(0)}\vA^{(1)} \cdots \vA^{(\cL)}\vS^{(\cL)}\|_F \label{eq:hNMFfroerr}
\end{equation}
necessarily increases as error propagates with each step.  As a result, significant error is introduced when $\cL$ is large.  Choosing $k^{(0)}, k^{(1)}, ... k^{(\cL)}$ in practice proves difficult when the number of topics at each layer is unknown, as the number of possibilities grow combinatorially.  Additionally, large differences between the number of topics for adjacent layers introduces large error into the factorization.  To alleviate the error propagation between layers, HNMF can be endowed with the structure of a neural network. The next sections focus on developing those extensions, inspired by \cite{FH18}.

Finally, as the NMF problem is ill-posed and has an infinite number of global minima (e.g., one may appropriately rescale each of the factors), this ill-posedness is exacerbated in HNMF.  There are unicity results for NMF when the matrix $\vX$ satisfies specific constraints (see e.g., \cite{donoho2004does, huang2013non, gillis2012sparse, fu2018nonnegative, laurberg2008theorems}), but we do not know of such results for the HNMF model; this is an important theoretical direction for future investigation.

\subsection{Deep NMF (DNMF)}\label{subsec:DNMF}

In \cite{FH18}, the authors make a first step toward bridging the gap between HNMF and neural networks.  Their method achieves a considerable performance improvement over standard NMF in classification.
The forward process for DNMF is HNMF with pooling operator, $p$, applied after each layer of decomposition to introduce nonlinearity and minimize overfitting.  Without the pooling operation, the DNMF model is identical to HNMF.  The other major contribution of \cite{FH18} is a proposed backpropagation algorithm meant to refine the result obtained from the forward process.  However, the backpropagation technique introduced in \cite{FH18} differs from backpropagation techniques in neural network settings, as it only propagates one layer at a time and uses multiplicative updates instead of gradient descent to update the values of $\vA$ and $\vS$.

\subsection{Other Related Work}\label{subsec:relatedwork}

Similar ideas were explored in \cite{trigeorgis2016deep}, \cite{le2015deep}, and \cite{SNT17}.  In \cite{trigeorgis2016deep}, the authors develop a hierarchical model in which some of the nonnegativity constraints are relaxed; however, this lacks our proposed backpropagation algorithm for training the model.  In \cite{le2015deep}, the authors propose a NMF backpropagation algorithm using an ``unfolding" approach; however, their method does not allow for hierarchy.  Finally, a method similar to ours was developed in \cite{SNT17}, but differs in that it lacks the nonnegativity constraints that makes our method applicable in topic modeling and feature extraction.

\section{Proposed Method: Neural NMF}\label{sec:proposedmethod}

The proposed backpropagation of DNMF \cite{FH18} behaves differently than the backpropagation technique typically employed by neural networks. Traditional backpropagation determines the gradient of a cost function with respect to \emph{all} the weights in the network, so that all the weights in the network may be updated at once. In \cite{FH18}, the update for $\vS^{(\ell)}$ only depends on
$\vA^{(\ell+1)}$ and $\vS^{(\ell+1)}$. One of the barriers to formulating a proper backpropagation step for DNMF is that in optimization methods like multiplicative updates \cite{lee2001algorithms} or alternating least squares \cite{paatero1994positive}, the $\vA$ and $\vS$ matrices take turns acting as the independent and dependent variables in the updates. This is in contrast to
neural networks,
where the weights connecting neurons between layers are always the independent variables, while the activations of the neurons are dependent on the weights. This separation of independent and dependent variables allows for calculation of the derivatives of the dependent variables with respect to the independent variables in a relatively simple way.
\begin{wrapfigure}{R}{0.6\textwidth}
\begin{minipage}[b]{0.6\textwidth}
\begin{algorithm}[H]
	\caption{Neural NMF}
	\label{alg:Neural NMF}
	\begin{algorithmic}
		\Require data matrix $\vX \in \mathbb{R}^{N \times M}$, number of layers $\cL$, step size $\gamma$, cost function $\loss$, initial matrices $\vA^{(i)}$ for $i = 0,..., \cL$
		\While{not converged}
		\State \text{ForwardPropagation}({$\vA^{(0)}, ..., \vA^{(\ell)}$})
		\For{$i := 0, ..., \cL$}
		\State $\vA^{(i)} \leftarrow \vA^{(i)} - \gamma * \frac{\partial \loss}{\partial \vA^{(i)}}$
		\State $\vA^{(i)} \leftarrow \vA^{(i)}_+$
		\EndFor
		\EndWhile
	\end{algorithmic}
\end{algorithm}
\end{minipage}
\end{wrapfigure}

To proceed, we choose to regard the $\vA$ matrices as the independent variables in our model. This is natural since the $\vS$ matrix at one layer is ``passed'' to the next layer for factorization, analogous to the neurons' activations in a neural network being passed to the next layer. In this analogy, the entries of the $\vA$ matrix now act as the weights between the neurons.

Since we choose to regard the $\vA$ matrices as the independent variables, we need to determine the $\vS$ matrices from the $\vA$ matrices. The natural way to do this is to always require the $\vS$ matrices to solve the nonnegative least squares problem (\ref{nnls}). Suppose $\vA^{(0)},\dots,\vA^{(\cL)}$ are given and define $\vS^{(-1)} = \vX$.  Then we let
\begin{equation} \label{nnls}
\begin{aligned}
\vS^{(\ell)} &= \argmin_{\vS \geq 0} \|\vS^{(\ell-1)} - \vA^{(\ell)}\vS\|, \quad \ell = 0,1,2,\cdots,\cL.
\end{aligned}
\end{equation}

We define the forward-propagation function $q(\vA,\vX)$,
for any nonnegative matrices $\vX$ and $\vA$ with the same number of rows, as the solution to the nonnegative least-squares problem between $\vX$ and $\vA$; that is,
\begin{equation} \label{lsqnonneg}
q(\vA,\vX) = \argmin_{\vS \geq 0} \|\vX - \vA\vS\|.
\end{equation}
We note that this problem is ill-posed if $\vA$ does not have full column rank. In what follows, we make the assumption that every $\vA$ matrix has full column rank. This assumption is reasonable, as the $\vA$ matrices are always tall
matrices (i.e., have more rows than columns), and this condition would only be violated in pathological circumstances. Using this notation, we can define the $\vS$ matrices as
\begin{equation*}
\begin{aligned}
\vS^{(0)} &= q(\vA^{(0)},\vX), \\ \vS^{(\ell)}&=q(\vA^{(\ell)},\vS^{(\ell-1)}), \quad \ell = 1,2,\cdots, \cL.
\end{aligned}
\end{equation*}
These equations show that $\vS^{(\ell_j)}$ depends on $\vA^{(\ell_i)}$ for $\ell_i \leq \ell_j$, but not for $\ell_i > \ell_j$
and they form the forward-propogation stage of Neural NMF. In Algorithms \ref{alg:Neural NMF} and \ref{alg:Forward Prop} we display the pseudocode for our proposed method. The partial derivatives $\frac{\partial \loss}{\partial \vA^{(i)}}$ will be derived in Theorem \ref{full_backprop}.
\begin{wrapfigure}{R}{0.6\textwidth}
\begin{minipage}[b]{0.6\textwidth}
\begin{algorithm}[H]
	\caption{Forward Propagation}
	\label{alg:Forward Prop}
	\begin{algorithmic}
		\Procedure{ForwardPropagation}{$\vA^{(0)}, ..., \vA^{(\ell)}$}
		\For{$i := 0,...,\cL$}
		\State $\vS^{(i)} \leftarrow q(\vA^{(i)}, \vS^{(i-1)})$
		\EndFor
		\EndProcedure
	\end{algorithmic}
\end{algorithm}
\end{minipage}
\end{wrapfigure}

In the case that partial label information is provided for the last layer, we can perform a semi-supervised HNMF by additionally calculating the supervision matrix $\vB$
as $\vB = (\vZ \odot \vY) (\vS^{(\cL})^\dagger$ where $\vY$ and $\vZ$ are the label and indicator matrices, respectively, defined in Section
\ref{subsec:SSNMF}.
One can include a term in loss function $\loss$ encouraging matrix similarity between $\vY$ and $\vB\vS^{(\cL)}$, such as $\|\vZ \odot (\vY-\vB\vS^{(\cL)}) \|$, which will influence the learned $\vA$ and $\vS$ matrices via backpropagation.

\subsection{Backpropagation}\label{subsec:backwardpropogation}

In order to derive a backpropagation update for the matrices $\vA$, we differentiate
a cost function which depends on both the $\vA$ and $\vS$ matrices, $\loss = f(\vX,\vS^{(0)},\ldots,\vS^{(\cL)}, \vA^{(0)},\ldots,\vA^{(\cL)})$, with respect to the $\vA$ matrices. A natural choice for the cost $\loss$ is the HNMF error (\ref{eq:hNMFfroerr}), however our method can be employed for other choices of cost functions.  To differentiate, we employ the chain rule, differentiating $q(\vA,\vX)$ with respect to both $\vA$ and $\vX$.  See Figure~\ref{fig:CompGraph} for a visualization of the computational graph for the forward propagation.

\begin{figure}
\begin{scaletikzpicturetowidth}{\textwidth}
\begin{tikzpicture}[line cap=round,line join=round,>=triangle 45,x=1.0cm,y=1.0cm,scale=\tikzscale]
\clip(-1.3,7.66) rectangle (18,13.94);
\draw (-1.38,10.5) node[anchor=north west] {$\mathbf{A}^{(0)}$};
\draw (-1.23,11.32) node[anchor=north west] {$\mathbf{X}$};
\draw (0.42,11.3) node[anchor=north west] {$q(\cdotp,\cdotp)$};
\draw [->,line width=1.pt] (-0.58,11.) -- (0.36,11.);
\draw (2.94,10.5) node[anchor=north west] {$\mathbf{A}^{(1)}$};
\draw (4.79,11.32) node[anchor=north west] {$q(\cdotp,\cdotp)$};
\draw [line width=1.pt] (1.1,11.) circle (0.6512296062065976cm);
\draw (13.2,11.39) node[anchor=north west] {$\mathbf{S}^{(\cL)}$};
\draw [->,line width=1.pt] (-0.5,10.32) -- (0.38,10.78);
\draw (2.86,11.39) node[anchor=north west] {$\mathbf{S}^{(0)}$};
\draw (7.02,11.2) node[anchor=north west] {$\mathbf{\cdotp\cdotp\cdotp}$};
\draw (8.97,10.5) node[anchor=north west] {$\mathbf{A}^{(\cL)}$};
\draw (10.94,11.32) node[anchor=north west] {$q(\cdotp,\cdotp)$};
\draw (8.5,11.39) node[anchor=north west] {$\mathbf{S}^{(\cL-1)}$};
\draw [->,line width=0.6pt] (6.24,11.) -- (7.06,11.);
\draw [line width=0.6pt] (5.45,11.) circle (0.6512296062065979cm);
\draw [line width=0.6pt] (11.61,11.) circle (0.6512296062065961cm);
\draw [->,line width=0.6pt] (7.74,11.) -- (8.56,11.);
\draw [->,line width=0.6pt] (3.72,11.02) -- (4.66,11.02);
\draw [->,line width=0.6pt] (1.9,11.) -- (2.84,11.);
\draw [->,line width=0.6pt] (10.,10.32) -- (10.88,10.78);
\draw [->,line width=0.6pt] (3.8,10.32) -- (4.68,10.78);
\draw [->,line width=0.6pt] (12.42,11.) -- (13.36,11.);
\draw [->,line width=0.6pt] (9.9,11.) -- (10.84,11.);
\draw [->,line width=0.6pt] (16.82,10.98) -- (17.56,10.98);
\draw (17.62,11.25) node[anchor=north west] {$\loss$};
\draw [line width=1.pt] (15.95,11) circle (0.6512296062065961cm);
\draw (15.43,11.32) node[anchor=north west] {$f( \; \cdotp)$};
\draw [->,line width=1.pt] (14.22,11.) -- (15.16,11.);
\draw [shift={(12.66,6.82)},line width=1.pt]  plot[domain=1.1840758899039464:2.0509017598668096,variable=\t]({1.*6.4153245825147195*cos(\t r)+0.*5.8153245825147195*sin(\t r)},{0.*5.8153245825147195*cos(\t r)+1.*5.2153245825147195*sin(\t r)});
\draw [shift={(9.91,-0.6)},line width=1.pt]  plot[domain=1.2558293757493167:1.947084517671192,variable=\t]({1.*17.399801525214039*cos(\t r)+0.*12.999801525214039*sin(\t r)},{0.*13.099801525214039*cos(\t r)+1.*13.099801525214039*sin(\t r)});
\draw [shift={(8.06,-0.23)},line width=1.pt]  plot[domain=1.1138293757493167:2.121084517671192,variable=\t]({1.*16.899801525214039*cos(\t r)+0.*13.099801525214039*sin(\t r)},{0.*13.099801525214039*cos(\t r)+1.*13.699801525214039*sin(\t r)});
\draw [shift={(12.,19.)},line width=1.pt]  plot[domain=4.49134191770967:5.058356252997486,variable=\t]({1.*9.121951545584968*cos(\t r)+0.*9.121951545584968*sin(\t r)},{0.*9.121951545584968*cos(\t r)+1.*9.121951545584968*sin(\t r)});
\draw [shift={(9.3,23.48)},line width=1.pt]  plot[domain=4.318438461516672:5.1277981782075575,variable=\t]({1.*14.641543634466961*cos(\t r)+0.*14.641543634466961*sin(\t r)},{0.*14.641543634466961*cos(\t r)+1.*14.641543634466961*sin(\t r)});
\draw [shift={(7.5,25.2)},line width=1.pt]  plot[domain=4.233693207929225:5.209011803811411,variable=\t]({1.*17.237563632949986*cos(\t r)+0.*17.237563632949986*sin(\t r)},{0.*17.237563632949986*cos(\t r)+1.*17.237563632949986*sin(\t r)});
\draw [->,line width=0.6pt] (15.08,11.62) -- (15.34,11.48);
\draw [->,line width=0.6pt] (15.208805189558234,10.083701211459573) -- (15.58,10.26);
\draw [->,line width=0.6pt] (15.649500719550645,10.01055504562318) -- (15.94,10.2);
\draw [->,line width=0.6pt] (15.1,11.92) -- (15.4,11.75);
\draw [->,line width=0.6pt] (15.4,12.12) -- (15.7,11.95);
\draw [->,line width=0.6pt] (14.949864727097806,10.368181067015819) -- (15.32,10.52);
\end{tikzpicture}
\end{scaletikzpicturetowidth}
    \caption{Computational graph of the forward propagation step of the algorithm for a cost function $\loss = f(\vX, \vS^{(0)},\ldots,\vS^{(\cL)}, \vA^{(0)},\ldots,\vA^{(\cL)})$. We use this graph to guide our back propagation step by tracing back all paths from each variable to the cost, $\loss$. }
    \label{fig:CompGraph}
\end{figure}

We begin by computing formulas for the derivatives of $q(\vA,\vX)$ with respect to $\vA$ and $\vX$.
Lemma~\ref{q_acts_columnwise} (presented in the appendix), shows that $\frac{\partial q}{\partial X}(\vA,\vX)$
can be computed columnwise (with respect to columns of $\vX$), so we need only determine derivative formulas for $q(\vA,\vx)$, where $\vx$ is a column of matrix $\vX$.  We define a space in which differentiation of $q(\vA, \vX)$ is relatively simple below.

\begin{definition}\label{def:Urs}
Let $U_{r,s} \subset \mathbb{R}_+^{r \times s} \times \mathbb{R}_+^{r}$ ($r \geq s$) be the set of all matrix-vector pairs
$(\vA,\vx)$ such that $\vA$ has full (column) rank and the support of $q$ does not change in some neighborhood of $(\vA,\vx)$.
\end{definition}

\begin{lemma} \label{measure_zero}
The set defined in Definition \ref{def:Urs}, $U_{r,s}$ is dense, open, and has full (Lebesgue) measure.
\end{lemma}

Roughly, this lemma says that most pairs $\vA$ and $\vx$ encountered in the course of our method will be such that the support of the computed column of $S$ will not change; that is, we can use supports learned during forward propagation for backpropagation gradients.
We delay the somewhat technical proof of this lemma to Appendix \ref{proofs}. Now we utilize Lemma~\ref{formula_for_q} to differentiate $q(\vA,\vx)$.

\begin{theorem} \label{derivs_of_q}
Let $U_{r,s}$ be the set defined in Definition \ref{def:Urs}, and suppose $(\vA,\vx) \in U_{r,s}$. Let $T = \supp q(\vA,\vx)$. The partial derivative of $q$ with respect to $\vx$ is given entry-wise as
\begin{equation} \label{X_deriv_of_q}
\left(\frac{\partial q}{\partial \vx}(\vA,\vx)\right)_{T,:} = \vA_{:,T}^\dagger , \quad\text{  and  }\quad \left(\frac{\partial q}{\partial \vx}(\vA,\vx)\right)_{T^c,:} = 0.
\end{equation}
The partial derivative of $q$ with respect to row $i$ of $\vA$ for $i \in [r]$ is given entry-wise as
\begin{equation}\label{A_deriv_of_q}
\left(\frac{\partial q}{\partial \vA_{i,:}}(\vA,\vx)\right)_{T,T} = -\left(\vA_{:,T}^\dagger \right)_{:,i}\left(\vA_{:,T}^\dagger  \vx\right)^\top + \left(\left(\vI-\vA_{:,T}\vA_{:,T}^\dagger \right)\vx\right)_{i}\vA_{:,T}^\dagger \left(\vA_{:,T}^\dagger \right)^\top, \;\text{  and  }\; \left(\frac{\partial q}{\partial \vA_{i,:}}(\vA,\vx)\right)_{(T \times T)^c} = 0.
\end{equation}
\end{theorem}

These formulas for the derivatives of $q$ are sufficient for implementing a backwards propagation algorithm in a machine learning library, such as PyTorch or TensorFlow. We summarize in Theorem~\ref{full_backprop} how these partial derivatives are applied via the chain rule for an arbitrary differentiable cost function $\loss$.  Assume there are given values for the variables $\vA^{(0)}, \cdots, \vA^{(\cL)}$ and the dependent (matrix) variables $\vS^{(\ell)}$ are defined as functions of the independent (matrix) variables $\vA^{(\ell)}$ recursively as
\begin{equation} \label{S_def}
\vS^{(0)} = q(\vA^{(0)},\vX), \quad\text{ and }\quad \vS^{(\ell)} = q(\vA^{(\ell)},\vS^{(\ell-1)}),
\end{equation} for $\ell \in [\mathcal{L}]$, where $\mathcal{L}$ is the number of layers.

This calculation requires one to collect derivatives along all paths between $\loss$ and the variable of interest in the computational graph in Figure~\ref{fig:CompGraph}.  We achieve this via the use of several auxiliary variables which collect partial derivatives along different segments of the paths.  First, the variable $\Phi$ collects (column-wise) derivatives along the central path between $\vS^{(\ell_2)}$ and $\vS^{(\ell_1)}$.

\begin{definition}\label{def:phi}
Let $T_m^{(\ell)} = \supp \vS^{(\ell)}_{:,m}$ where $m \in [k^{(\ell-1)}]$ is a column index of $\vS^{(\ell)}$ which has $k^{(\ell-1)}$ columns.
Define $\vPhi^{(\ell_1,\ell_2),m}$ for $\cL \geq \ell_2 \geq \ell_1 \geq 0$ by
$$\vPhi^{(\ell_1,\ell_2),m} = \left({\vA^{(\ell_2)}_{:,T_m^{(\ell_2)}}}^{\dagger}\right)_{:,T_m^{(\ell_2-1)}}\left({\vA^{(\ell_2-1)}_{:,T_m^{(\ell_2-1)}}}^{\dagger}\right)_{:,T_m^{(\ell_2-2)}}\dots \left({\vA^{(\ell_1+1)}_{:,T_m^{(\ell_1+1)}}}^{\dagger}\right)_{:,T_m^{(\ell_1)}}\left({\vA^{(\ell_1)}_{:,T_m^{(\ell_1)}}}^{\dagger}\right).$$
\end{definition}

The variable $\mathbf{d}$ utilizes $\Phi$ to collect (column-wise) derivatives along the path that follows the edge from $\loss$ back to $\vS^{(\ell_2)}$, and along the central path between $\vS^{(\ell_2)}$ to $\vS^{(\ell_1)}$.  Finally, the variable $\mathbf{U}$ adds to this path the edge between $\vS^{(\ell_1)}$ and $\vA^{(\ell_1)}$, which is the variable of interest.

\newcommand\holdconstant{\textbf{*}}

\begin{definition}
Suppose $\loss$ is a cost function depending on all the variables $\vS^{(\ell)}$ and $\vA^{(\ell)}$. Let $\left(\frac{\partial \loss}{\partial \vS^{(\ell_2)}} \right)^\holdconstant$ denote the derivative of $\loss$ with respect to $\vS^{(\ell_2)}$ holding $\vS^{(\ell_2+1)}$, $\ldots$ , $\vS^{(\cL)}$ constant.  We define
$$\vd^{(\ell_1,\ell_2),m} = \left({\vPhi^{(\ell_1,\ell_2),m}}\right)^\top \left(\frac{\partial \loss}{\partial \vS^{(\ell_2)}}\right)^\holdconstant_{T_m^{(\ell_2)},m},$$
and $\vU^{(\ell_1,\ell_2),m}$ with
$$\vU^{(\ell_1,\ell_2),m}_{:,T_m^{(\ell_1)}} = - \vd^{(\ell_1,\ell_2),m}\left({\vS^{(\ell_1)}_{T_m^{(\ell_1)},m}}\right)^\top + \left(\vS^{(\ell_1 -1)}-\vA^{(\ell_1)}\vS^{(\ell_1)}\right)_{:,m} \left(\vd^{(\ell_1,\ell_2,),m}\right)^\top \left({\vA^{(\ell_1)}_{:,T_m^{(\ell_1)}}}^\dagger\right)^\top, \text{ and } \vU^{(\ell_1,\ell_2),m}_{:,{T_m^{(\ell_1)}}^c} = 0.$$ \label{def:dandU}
\end{definition}

The final desired derivative, $\frac{\partial \loss}{\partial \vA^{(\ell_1)}}$, then sums these collected derivatives over all paths between $\loss$ and $\vA^{(\ell_1)}$, including the direct edge.

\begin{theorem} \label{full_backprop}
Let the dependent (matrix) variables $\vS^{(\ell)}$ be defined as functions of the independent (matrix) variables $\vA^{(\ell)}$ recursively as
\begin{equation}
\vS^{(0)} = q(\vA^{(0)},\vX), \quad\text{ and }\quad \vS^{(\ell)} = q(\vA^{(\ell)},\vS^{(\ell-1)}),
\end{equation} for $\ell \in [\mathcal{L}]$, where $\mathcal{L}$ is the number of layers.
Let $T_m^{(\ell)}$ and $\vPhi^{(\ell_1,\ell_2),m}$ for $\cL \geq \ell_2 \geq \ell_1 \geq 0$ be as defined in Definition~\ref{def:phi}.
Fixing a point in the space of $\vA$-matrices, suppose that for all $0 \le \ell \le \mathcal{L}$ and $m \in [k^{(\ell -1)}]$, $(\vA^{(\ell)},\vS^{(\ell-1)}_{:,m})$ is in the set $U_{k^{(\ell-1)},k^{(\ell)}}$ defined in Definition \ref{def:Urs}.
Suppose $\loss$ is a cost function depending on all the variables $\vS^{(\ell)}$ and $\vA^{(\ell)}$. Let $\left(\frac{\partial \loss}{\partial \vS^{(\ell_2)}} \right)^\holdconstant$, $\vd^{(\ell_1,\ell_2),m}$, and $\vU^{(\ell_1,\ell_2),m}$ be as defined in Definition~\ref{def:dandU}.
Then, if we let $\left(\frac{\partial \loss}{\partial \vA^{(\ell_1)}} \right)^\vS$ be the derivative of $\loss$ with respect to $\vA^{(\ell_1)}$, holding the $\vS$ matrices constant, we have
$$\frac{\partial \loss}{\partial \vA^{(\ell_1)}} = \left(\frac{\partial \loss}{\partial \vA^{(\ell_1)}} \right)^{\vS} + \sum_{\substack{\ell_1 \leq \ell_2 \leq \cL \\ 1 \leq m \leq M}} \vU^{(\ell_1,\ell_2),m}. $$
\end{theorem}

That is,
if the matrices $\vA$ and $\vS$ are such that $(\vA^{(\ell)},\vS^{(\ell-1)}_{:,m}) \in U_{k^{(\ell-1)},k^{(\ell)}}$ of Definition~\ref{def:Urs}
then we can compute the derivative of differentiable cost function $\loss$ with respect to the independent variables $\vA$ using simple matrix operations and the support information calculated during forward propagation.

\section{Experimental Results}\label{sec:experimentalresults}
We test Neural NMF on three datasets: a small, synthetic dataset, the 20 Newsgroups dataset, and the MyLymeData dataset. The synthetic dataset is a small block matrix, with three different levels of hierarchy in the blocks.
The 20 Newsgroups dataset is a common benchmark dataset in which hierarchy between the topics of the dataset is known.  Meanwhile, the MyLymeData dataset represents the symptoms experienced by a group of Lyme disease patients, and hierarchy in this dataset is not known a priori. Our implementation is available in the indexed Python package \texttt{NeuralNMF} and the code for experiments is provided on \texttt{Github}~\cite{NeuralNMFpackage}.

On the synthetic dataset, we compare Neural NMF to HNMF and Deep NMF in the unsupervised, semi-supervised, and fully supervised settings with 1, 2, and 3 layers, and report classification accuracy and reconstruction loss. We see that in each of these settings, Neural NMF outperforms HNMF and Deep NMF in both classification and reconstruction, and forms a better low rank representation of the data set that appears to preserve more of the coarser block structure.

On the 20 Newsgroups dataset, we compare Neural NMF to HNMF in the unsupervised and semi-supervised settings for 2 layer experiments, and report classification accuracy at each layer. We focus on classification and qualitative analysis rather than reconstruction because on this data set, one cannot ask to produce a highly accurate low-rank reconstruction of the data but instead seek to form meaningful and class-discriminatory topics that form a hierarchical structure. Our experimental results show that for both the unsupervised and supervised settings, Neural NMF attains a higher classification accuracy than HNMF at each layer, and the topics formed by Neural NMF have significantly better interpretability and hierarchical structure. We also see that despite only being provided partial label information at the second layer, Neural NMF is able to improve the classification accuracy at the first layer with the additional of this label information, demonstrating that unlike HNMF, Neural NMF is able to propagate information provided at the last layer to earlier layers.

On the MyLymeData set, we see the potential for Neural NMF to produce hierarchical topic structure on a real-world large-scale survey dataset.  While the ground truth hierarchical structure is unknown for this real, messy dataset, we note that the results produced by Neural NMF yield interpretable results that reflect both what is well-known and unknown about Lyme disease patients and the manifestation of their symptoms.

\subsection{Synthetic Data}
We first test the reconstruction and classification ability of Neural NMF in an idealized setting: on a $90 \times 87$  noisy toy dataset with a clear three-layer hierarchical structure.  Starting with two large blocks, we overlay increasingly smaller and more intense asymmetric regions along the diagonal of a matrix, and finally add a $\text{uniform}(0,1)$ noise to the entire matrix; see the left plot of Figure \ref{fig:unsupervised two layer results}. We know the optimal model rank sequence to be $k^{(0)} = 9, k^{(1)} = 4, k^{(2)} = 2$ \textit{a priori}. We test HNMF, Deep NMF, and Neural NMF with one, two, or three layer structure, and various levels of supervision. The labels representing to which of the nine classes the data points belong (grouped by the highest intensity blocks) are given for 40\% of the data (semisupervised) or 100\% of the data (supervised). For each level of supervision and depth, the results are averaged over 25 trials.   We present the recovery error and classification accuracy measuring the discrepancy between $\vY$ and the computed matrices $\vB$ and $\vS$ for these experiments in Table \ref{tab:comparison of methods synthetic data}.  The reconstruction error is computed relative to the norm of the original matrix as $\|\vX - \vA^{(0)}\vA^{(1)} \cdots \vA^{(\cL)}\vS^{(\cL)}\| / \|\vX\|$.  Object $m$ is predicted to have label $p$ if $(\vB\vS^{(\cL)})_{pm} = \max$ $(\vB\vS^{(\cL)})_{:m}$, and classification accuracy is the proportion of predicted labels that match the true labels.   Noteworthy improvements of Neural NMF over HNMF and Deep NMF are bolded.  We comment that we expect the advantage Neural NMF enjoys over HNMF and Deep NMF is due to the backpropagation method allowing it to avoid suboptimal local minima found by HNMF and Deep NMF.  Additionally, in this example the classification accuracy is inherently tied to the reconstruction error (as the labels are generated by entries of the matrix), so Neural NMF is able to achieve good accuracy even in the unsupervised setting.

While the approximations produced at the $\ell$th layer have rank $k^{(\ell)}$, the rank of the final approximation produced by the hierarchical model will be $k^{(\mathcal{L})}$.  For this reason, we can only seek the $k^{(\mathcal{L})}$ most representative features (blocks) when qualitatively evaluating the final approximations produced by hierarchical models.

In Figure \ref{fig:unsupervised two layer results}, we visualize the reconstructions produced by each method with no supervision and two layer structure ($k^{(0)} = 9$ and $k^{(1)} = 4$). We cannot hope to resolve the highest-intensity features from the original data in our reconstructions, as the NMF approximations have lower than necessary rank.
Instead, we consider how accurately the methods reconstruct the two-layer block structure.  Although each method is able to capture some of the structure, it is clear that Neural NMF outperforms HNMF and Deep NMF, resolving sharper blocks.

In Figure \ref{fig: semisupervised three layer}, we see that Neural NMF similarly outperforms HNMF and Deep NMF in a semisupervised three-layer trial ($k^{(0)} = 9, k^{(1)} = 4, k^{(2)} = 2$); note that HNMF and Deep NMF produce many columns with extemely low intensity (HNMF entirely misses the second of the two coarsest blocks), while Neural NMF produces a rank-2 approximation which correctly
reconstructs the coarsest two-block structure of least intensity (lightest blue). We note that adding label information is expected to lead to worse reconstruction because the optimization task will focus on improving classification, often at the detriment of higher reconstruction loss.

\begin{table}
	\centering
	\caption{Reconstruction error / classification accuracy for various supervision levels and layer structures on the synthetic dataset. Entry marked with $*$ corresponds to the experiment represented in Figure \ref{fig:unsupervised two layer results}, and entry marked with $**$ corresponds to the experiment represented in Figure \ref{fig: semisupervised three layer}.}

	\label{tab:comparison of methods synthetic data}
	\begin{tabular}{| c | c | c | c | c |}
		\hline
		 & \hspace{-0.1cm}Layers\hspace{-0.1cm} & Hier. NMF & Deep NMF & Neural NMF \\
		\hline
		 \multirow{3}{*}{\hspace{-0.1cm}Unsuper.\hspace{-0.1cm}} & 1 & 0.053 / 0.111	   & 0.031 / 0.111		& 0.029 / \textbf{1} \\
		 & 2 & 0.399 / 0.222		  &  0.414 / 0.222       & \textbf{0.310} / \textbf{0.995} $^*$ \\
		 & 3 & 0.860 / 0.356 & 0.838 / 0.356	& \textbf{0.492} / \textbf{1}  \\

		\hline
		 \multirow{3} {*} {\hspace{-0.1cm}Semisuper.\hspace{-0.1cm}} & 1 & 0.049 / 0.933   & 0.031 / 0.947 & 0.042 / \textbf{1}  \\
		 & 2 & 0.374 / 0.926  & 0.394 / 0.911 & \textbf{0.305} / \textbf{1} \\
		 & 3 & 0.676 / 0.930 & 0.733 / 0.930 & \textbf{0.496} / \textbf{0.990} $^{**}$\\

		\hline
		 \multirow{3} {*} {\hspace{-0.1cm}Supervised\hspace{-0.1cm}} & 1  &  0.052 /  0.960  & 0.042 / 0.962 & 0.042 / \textbf{1} \\
		 & 2 & 0.311  / 0.984 & 0.310 / 0.984 & 0.307 / 1 \\
		 & 3 & 0.495 / 1 & 0.494 / 1& 0.498 / 1 \\
		\hline
	\end{tabular}
\end{table}

\begin{figure}
	\centering
	\includegraphics[width=0.8\textwidth]{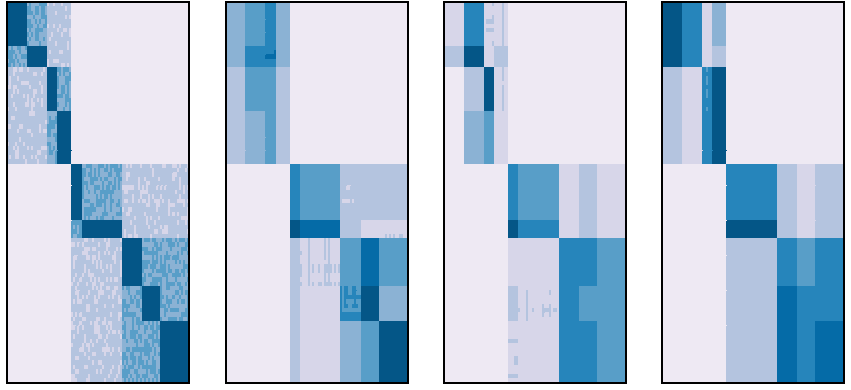}
	\caption{Rank 4 approximations of the original dataset when no label information is provided and a two-layer structure $k^{(0)} = 9, k^{(1)} = 4$ is specified, constructed by $\vX \approx \vA^{(0)}\vA^{(1)}\vS^{(1)}$.  Left to right: original data, HNMF approximation, Deep NMF approximation, and Neural NMF approximation.}
	\label{fig:unsupervised two layer results}
\end{figure}

\begin{figure}
	\centering
	\includegraphics[width=0.8\textwidth]{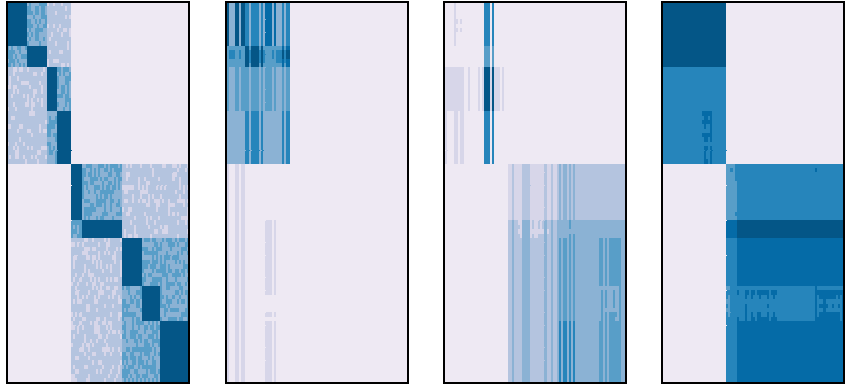}
	\caption{Rank 2 approximations of the original dataset when 40\% of the label information is specified as known and a three-layer structure $k^{(0)} = 9, k^{(1)} = 4, k^{(2)}=2$ is specified, constructed by $\vX \approx \vA^{(0)}\vA^{(1)}\vA^{(2)}\vS^{(2)}$.  Left to right: original data, HNMF approximation, Deep NMF approximation, and Neural NMF approximation.}
	\label{fig: semisupervised three layer}
\end{figure}

\subsection{20 Newsgroups Data}

The 20 Newsgroups dataset is a collection of approximately 20,000 text documents containing the text of messages from 20 different newsgroups on the distributed discussion system Usenet which functioned similarly to current internet discussion forums.  The documents are partitioned nearly evenly across the 20 newsgroups which can be further classified into six supergroups (computers, for sale, sports/recreation, politics, science, religion) \cite{KL08}.  This clear hierarchical topic structure makes this an appropriate testing ground for Neural NMF.

In our experiments, we subsample 100 documents from each of 10 subtopics from the 20 Newsgroups dataset. We encode the text data in a word frequency vector representation of the bag-of-words model. We perform experiments for both the semi-supervised and unsupervised tasks.  For semi-supervision, we provide labels for 75\% of the documents and compute classification accuracy on the 25\% of documents without provided labels. We present the keywords from each experiment (those words which are represented with largest magnitude in each topic) and the classification accuracy measuring the discrepancy between $\vY$ and the computed matrices $\vB$ and $\vS$. The reconstruction error is computed as $\|\vX - \vA^{(0)}\vA^{(1)} \cdots \vA^{(\cL)}\vS^{(\cL)}\| / \|\vX\|$.  Object $m$ is predicted to have label $p$ if $(\vB\vS^{(\cL)})_{pm} = \max$ $(\vB\vS^{(\cL)})_{:m}$, and classification accuracy is the proportion of predicted labels that match the true labels.

In Table \ref{tab:20newssub}, we display the classifications accuracies for Neural NMF, Deep NMF, and HNMF for the first and second layers of an unsupervised and semi-supervised 2 layer experiment, where supervision labels are provided only for the last layer. Each experiments was run for ten trials and we report the average of the trials. We see that Neural NMF outperforms HMF
in all setting, and outerperforms Deep NMF when supervision information is provided. Comparing the unsupervised and supervised experiments, we see that even though label information is provided only at the second layer, the first layer of Neural NMF gains substantial improvement when supervision is added, suggesting that the classification information at the second layer successfully propagated to the first layer. This is not possible for HNMF, where each layer's factorization is computed separately.  We also see that for Deep NMF the first layer does not have significant improvement when supervision information is added.

 In Tables \ref{tab:keywords_10} and \ref{tab:keywords_6}, we display the keywords learned by Neural NMF at the first and second layers, respectively, for the semi-supervised two layer experiment. We see that the words are meaningful and representative of the topics within the 20 Newsgroups dataset, and we note that at rank 6, each of the six topics related directly to one of the six super topics of this dataset. We are also able to see clear hierarchical structure, such as a medical topic and space topic at the first layer (topics 5 and 7 in Table \ref{tab:keywords_10}) that combine into a science topic in the second layer (topic 4 in Table \ref{tab:keywords_6}). This hierarchical relationship is also evident from Figure \ref{fig:a2_matrix}, where we display a heat map of the $A_2$ matrix for the 2 layer semi-supervised Neural NMF experiment. We see a clear relationship between topics at rank 10 and rank 6, which agrees with the known hierarchy. The topic labels at rank 10 were chosen qualitatively based on the keywords seen in Tables \ref{tab:keywords_10}, and the topic labels at rank 6 were determined by the presence of 20 Newsgroups data set classes in the rows of the $\vS^{(1)}$ matrix.\footnote{At first glance, the keywords in Topic 6 of layer 1 do not appear to correspond to any 20 Newsgroups document topic, but all 10 keywords come from the email signature of a prolific user within medicine newsgroups with the name Gordon Banks, who includes his radio call sign N3JXP, his email \texttt{geb@cadre.dsl.pitt.edu}, and the quote ``Skepticism is the chastity of the intellect, and it is shameful to surrender it too soon" (see e.g.,~\cite[Pg. 259]{coelho2018building}).}

 In Tables \ref{tab:ssnmf_keywords_10} and \ref{tab:ssnmf_keywords_6} we display the keywords learned by HNMF at the first and second layers, respectively, for the semi-supervised two-layer experiment, and in Figure \ref{fig:a2_matrix} we display a heat map of the $\vA_2$ matrix for this experiment which shows the relationship between the topics at each layer. We see that while most of the topics provide salient topic modelling information corresponding to the 20 Newsgroups topics, some topics are fairly unclear (e.g., Topic 6, 7 of rank 10) and the learned hierarchical structure does not adhere well to the expected structure (e.g., baseball and motorcycles do not collect into the recreation super-topic, medicine and space do not collect into the science super-topic).

\begin{table}
	\centering
	\caption{Classification accuracies of each layer given for a two layer unsupervised experiment and a two layer semisupervised experiment on the subsampled 20 Newsgroups dataset. }
	\label{tab:20newssub}
	\begin{tabular}{| c | c | c | c | c | c |}
		\hline
		& \hspace{-0.1cm}Layer\hspace{-0.1cm} & Hier. NMF & Deep NMF & Neural NMF \\
		\hline
		\multirow{2}{*}{\hspace{-0.1cm}Unsuper.\hspace{-0.1cm}} & 1 & 0.593 & \textbf{0.638} & 0.604\\
		& 2 & 0.507 & 0.444 & \textbf{0.532} \\
		\hline
		\multirow{2}{*}{\hspace{-0.1cm}Semisuper.\hspace{-0.1cm}} & 1 & 0.593 & 0.642 & \textbf{0.690} \\
		& 2 & 0.546 & 0.536 & \textbf{0.654} \\

		\hline
	\end{tabular}
\end{table}

\begin{table}[htbp]
	\centering
	\caption{Topic keywords for layer 1 of the subsampled 20 Newsgroups dataset produced by Neural NMF.}
	\resizebox{\textwidth}{!}{
	\begin{tabular}{|c|cccccccccc|}
		\hline
		& Topic 1 & Topic 2 & Topic 3 & Topic 4 & Topic 5 & Topic 6 & Topic 7 & Topic 8 & Topic 9 & Topic 10 \\
		\hline
1 & drive & new & sale & bike & msg & geb & space & humanist & people & jesus \\
2 & apple & dj & offer & dod & food & pitt & launch & article & gun & god \\
3 & video & following & drive & motorcycle & used & banks & nasa & politics & israel & people \\
4 & mhz & st & mb & helmet & object & gordon & shuttle & omran & fbi & bible \\
5 & card & computer & best & games & dietz & n3jxp & moon & bedouin & government & christians \\
6 & graphics & fm & color & stadium & disease & chastity & ideas & backcountry & us & israel \\
7 & mac & mower & software & got & responses & cadre & orbit & speaking & say & christian \\
8 & sound & sold & disks & baseball & patients & skepticism & lunar & liar & arab & religion \\
9 & powerbook & battery & game & year & diet & dsl & centaur & absood & jews & order \\
10 & know & remains & shipping & players & epilepsy & shameful & medical & john & killed & rosicrucian \\
		\hline
	\end{tabular}%
	}
	\label{tab:keywords_10}%
\end{table}

\begin{table}[htbp]
	\centering
	\caption{Topic keywords for layer 2 of the subsampled 20 Newsgroups dataset produced by Neural NMF.}
	\resizebox{0.8\textwidth}{!}{
	\begin{tabular}{|c|cccccc|}
		\hline
		& Topic 1 & Topic 2 & Topic 3 & Topic 4 & Topic 5 & Topic 6 \\
		\hline
1 & drive & sale & bike & space & people & jesus \\
2 & apple & new & dod & know & gun & god \\
3 & video & computer & motorcycle & msg & israel & people \\
4 & graphics & offer & helmet & launch & article & bible \\
5 & mhz & drive & stadium & diet & fbi & israel \\
6 & card & shipping & got & cost & government & christians \\
7 & sound & mb & games & disease & guns & christian \\
8 & mac & sell & baseball & heard & us & religion \\
9 & powerbook & color & players & used & mr & jews \\
10 & projector & best & uhhhh & centaur & dear & rosicrucian \\
		\hline
	\end{tabular}%
	}
	\label{tab:keywords_6}%
\end{table}

\begin{figure}
	\centering
	\caption{Heatmap of the $A_2$ matrix for an experiment on the subsampled 20 Newsgroups dataset produced by Neural NMF, which illustrates how six supertopics are formed by linearly combining ten subtopics.}
	\includegraphics[width=.7\linewidth]{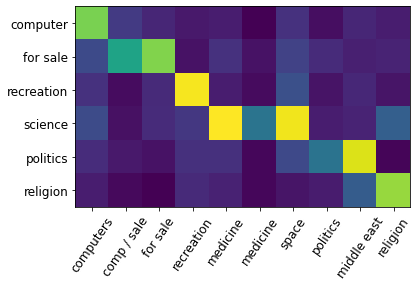}
	\label{fig:a2_matrix}
\end{figure}

\begin{table}[htbp]
	\centering
	\caption{Topic keywords for layer 1 of the subsampled 20 Newsgroups dataset produced by HNMF.}
	\resizebox{\textwidth}{!}{
	\begin{tabular}{|c|cccccccccc|}
		\hline
		& Topic 1 & Topic 2 & Topic 3 & Topic 4 & Topic 5 & Topic 6 & Topic 7 & Topic 8 & Topic 9 & Topic 10 \\
		\hline
1 & mail & mb & sale & team & geb & know & fbi & people & israel & jesus \\
2 & modem & drive & games & runs & pitt & files & space & gun & arab & god \\
3 & mac & card & offer & win & gordon & program & koresh & hudson & israeli & christians \\
4 & keyboard & ram & shipping & games & banks & motorcycle & government & us & lebanon & bible \\
5 & internal & vga & game & year & dsl & bike & fire & say & soldiers & christian \\
6 & apple & hard & power & pitching & chastity & postscript & time & guns & peace & christ \\
7 & computer & video & best & game & cadre & format & launch & moral & arabs & life \\
8 & use & floppy & system & last & njxp & file & handheld & morality & jews & law \\
9 & cable & color & sound & fans & skepticism & question & jmd & way & lebanese & love \\
10 & software & mhz & super & rbi & shameful & dod & com & data & occupied & jews \\
		\hline
	\end{tabular}%
	}
	\label{tab:ssnmf_keywords_10}%
\end{table}

\begin{table}[htbp]
	\centering
	\caption{Topic keywords for layer 2 of the subsampled 20 Newsgroups dataset produced by Hierarchical.}
	\resizebox{0.8\textwidth}{!}{
	\begin{tabular}{|c|cccccc|}
		\hline
		& Topic 1 & Topic 2 & Topic 3 & Topic 4 & Topic 5 & Topic 6 \\
		\hline
1 & mb & sale & team & geb & israel & jesus \\
2 & drive & games & runs & pitt & arab & god \\
3 & mail & offer & win & gordon & israeli & christians \\
4 & modem & shipping & year & banks & people & bible \\
5 & mac & game & games & dsl & government & christian \\
6 & know & power & pitching & chastity & lebanon & christ \\
7 & keyboard & system & game & cadre & fbi & life \\
8 & computer & best & last & njxp & soldiers & law \\
9 & internal & disks & fans & skepticism & peace & love \\
10 & apple & sound & rbi & shameful & arabs & jews \\
		\hline
	\end{tabular}%
	}
	\label{tab:ssnmf_keywords_6}%
\end{table}

\begin{figure}
	\centering
	\caption{Heatmap of the $A_2$ matrix for an experiment on the subsampled 20 Newsgroups dataset produced by HNMF, which illustrates how six supertopics are formed by linearly combining ten subtopics.}
	\hspace*{-1.8cm}
	\includegraphics[width=.7\linewidth]{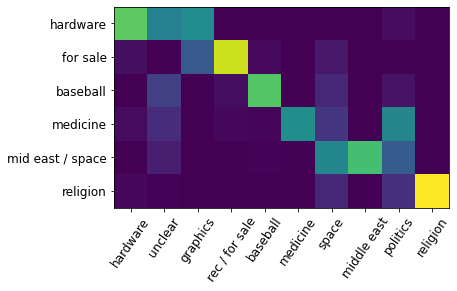}
	\label{fig:ssnmf_a2_matrix}
\end{figure}

\subsection{Lyme Data}
We conclude with a real world example using Lyme disease data. The MyLymeData dataset \cite{LymeData} used in this analysis consists of survey responses of approximately 4000 current and former Lyme disease patients (which has since grown to over 17,000).
The questions cover demographics, symptoms at various stages of the disease, medical procedures, and more. Responses may take binary, categorical, or scalar values.  Each patient is self-identified as `well' or `unwell' at the time of the survey.
We center our analysis on a subset of the dataset concerning patient symptom information shortly after an initial tick bite and at the time of diagnosis.  These questions yield binary data indicating whether or not, for example, a patient observed a `bulls-eye rash' somewhere on his or her body at the time of possible diagnosis.
As we do not know the hierarchy of this data a priori, we simply run exploratory experiments in this section.  We highlight potential advantages of Neural NMF illustrated in these results.

\begin{figure}
\centering
\includegraphics[width=\textwidth]{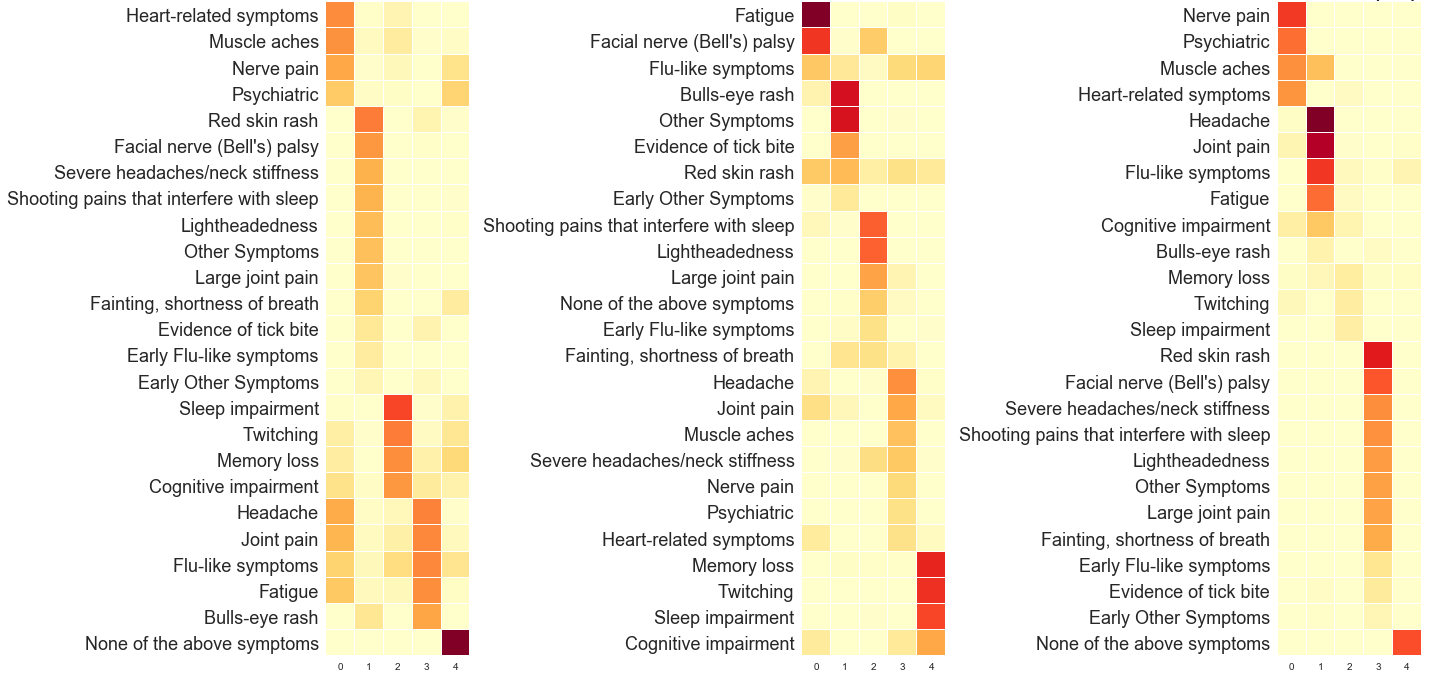}
\caption{The resulting S matrices from running unsupervised versions of NMF (left), HNMF (middle), and Neural NMF (right) on the full Lyme symptom dataset.  The NMF plot has model rank $k^{(0)} = 5$, while the plots for HNMF and Neural NMF are the visualizations of the layer with model rank $k^{(1)} = 5$ after running each method with layer structure $k^{(0)} = 6, k^{(1)} = 5$, and $k^{(2)} = 4$.}
\label{fig:nmf vs HNMF vs backprop combined symptom data 6 topics}
\end{figure}

Figure \ref{fig:nmf vs HNMF vs backprop combined symptom data 6 topics} is an example of the $\vS$ matrix when NMF is performed with model rank $k = 5$ alongside the $\vS$ matrices at the second layer for HNMF and Neural NMF with full network structure $k^{(0)} = 6, k^{(1)} = 5$, and $k^{(2)} = 4$.
First, note in Figure~\ref{fig:nmf vs HNMF vs backprop combined symptom data 6 topics} that at the second layer, Neural NMF reveals topics that are \textit{extremely} similar to those produced by simple NMF.
The Neural NMF topic structure reveals in particular an interesting observation about the appearance of a bulls eye rash, previously thought to be a critical element to a Lyme disease diagnosis. However, in the Neural NMF topic structure, this symptom actually plays very little role in any of the topics, versus appearing fairly strongly in the other two structures. Further investigation here is critical, as physicians are now beginning to agree that the prevalence (and thus importance) of the bulls eye rash symptom may be seriously less than previously believed.

\begin{figure}
\centering
\includegraphics[width=.8\textwidth]{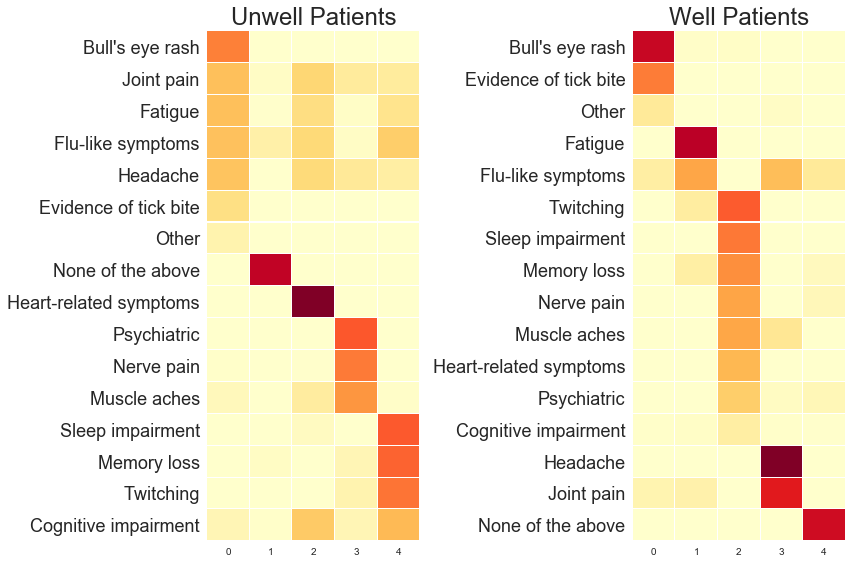}
\caption{The resulting $S$ matrices from applying Neural NMF with layer structure $k^{(0)} = 6$ and $k^{(1)} = 5$ to matrices containing data from unwell and well patients about symptoms at the time of first diagnosis. }
\label{fig:well vs unwell symptoms at diagnosis results}
\end{figure}

To further hone this observation, we present the results of Neural NMF with $k^{(0)} = 6$ and $k^{(1)} = 5$ on a subset of the data corresponding to symptoms experienced by patients at diagnosis but for unwell and well patients separately.  This is shown in Figure \ref{fig:well vs unwell symptoms at diagnosis results}, where we see drastically different positioning of the bulls eye rash symptom. In well patients, it forms a very strong topic indicating a prevalence in that patient group.  However, in unwell patients, it is mildly represented among a topic seeming to indicate general malaise. This warrants further investigation again, as this rash may play a role in whether a patient becomes well or not.
A final interesting observation is that twitching, a physical symptom, appears in the same topic as cognitive symptoms such as sleep impairment and memory loss.  These are all symptoms of a so-called neurological Lyme disease manifestation that is still not yet understood.  In summary, studying an accurate hierarchical topic structure via Neural NMF leads to important and interesting directions of further study both mathematically and medically.

\section{Conclusion}\label{sec:conclusion}

We present a novel method for hierarchical multilayer nonnegative matrix factorization which incorporates the backwards propagation technique from deep learning to minimize error accumulation.
Preliminary tests on toy datasets show this method outperforms existing multilayer NMF algorithms. The forward and backwards propagation steps of Neural NMF may offer decreased reconstruction and classification error over single-application HNMF.  Additionally, it seems that Neural NMF often better resolves data points into a single topic.

Future directions include to further compare Neural NMF and others on various datasets in order to find precise regimes in which it offers substantial improvement.  Furthermore, theoretically analyzing the convergence of Neural NMF on ideal datasets (those containing latent hierarchy) is an important future direction of work.

\section{Acknowledgements}
The authors would like to thank LymeDisease.org for the use of data derived from MyLymeData to conduct our experiments and the patients for their contributions to MyLymeData.  They additionally thank Dr.\ Anna Ma for her instruction on the MyLymeData dataset, LymeDisease.org CEO Lorraine Johnson for her collaboration, Dr.\ Blake Hunter for proposing the project, and Professor Andrea Bertozzi, Director of Applied Math at UCLA, for organizing the REU program through which this research was conducted.

\bibliography{bib}
\bibliographystyle{acm}

\appendix \label{sec:appendix}
\section{Proofs of Lemma \ref{q_cts} and Proposition \ref{measure_zero}} \label{proofs}

First, we state two lemmas about $q$, which will lead to the derivative formula. Our first lemma demonstrates that $q$ actually acts column-wise on $\vX$.
\begin{lemma} \label{q_acts_columnwise}
	Suppose $\vA \in \mathbb{R}_+^{N \times k}$ and $\vX \in \mathbb{R}_+^{N \times M}$ are nonnegative matrices (with the same number of rows), and $\vA$ has full column rank. Then, for all $m$,
	$$q_{:,m}(\vA,\vX) = q(\vA,\vX_{:,m}).$$
\end{lemma}

This lemma is a consequence of the column-wise separability of the Frobenius norm objective
and simplifies the problem of finding the derivatives of $q(\vA,\vX)$ for general $\vA$ and $\vX$ to the problem of finding the derivatives of $q(\vA,\vx)$ where $\vx$ is a column vector of size $N$. The following lemma now gives us a formula for $q(\vA,\vx),$ which we will be able to differentiate.

\begin{lemma} \label{formula_for_q}
	Suppose $\vA \in \mathbb{R}_+^{r,s}$ has full column rank, and $\vx \in \mathbb{R}_+^r$. Let $T$ be the support of $q(\vA,\vx)$, i.e., the set of indices on which the vector is strictly positive. Then $q(\vA,\vx)$ is given by
	$$q_T(\vA,\vx) = \vA_{:,T}^\dagger  \vx, \qquad  q_{T^c}(\vA,\vx) = 0.$$
\end{lemma}

\begin{proof}
	Since $q(\vA,\vx) = \argmin_{\vs \geq 0} \|\vx - \vA\vs\|$, for any nonnegative vector $\vs$,
	$$\|\vx-\vA\vs\| \geq \|\vx - \vA q(\vA,\vx)\|.$$
	Then, using that $q(\vA,\vx)$ has support $T$ (and so is zero outside that support), we get
	$$\|\vx - \vA q(\vA,\vx)\|= \|\vx - \vA_{:,T} q_T(\vA,\vx) - \vA_{:,T^c} q_{T^c}(\vA,\vx)\| = \|\vx - \vA_{:,T} q_T(\vA,\vx)\|,$$
	so
	$$\|\vx-\vA\vs\| \geq \|\vx - \vA_{:,T} q_T(\vA,\vx)\|$$
	for any nonnegative vector $\vs$. Suppose $\vw$ is in $(0,\infty)^{|T|}$, i.e., it is a positive vector of length $|T|$. Let $\vv$ be the vector in $\mathbb{R}_+^r$ with $\vv_T = \vw$ and $\vv_{T^c} =0$. Then
	$$\|\vx - \vA_{:,T}\vw\| = \|\vx - (\vA_{:,T}\vv_T + \vA_{:,T^c}\vv_{T^c}) \| = \|\vx - \vA \vv\| \geq \|\vx -\vA_{:,T}q_T(A,x) \|.$$
	This holds for any $\vw$ in $(0,\infty)^{|T|}$.

	Since $q(\vA,\vx)$ has support equal to $T$, all of the entries of $q_T(\vA,\vx)$ are positive, so it is in $(0,\infty)^{|T|}$. So, by what we just showed,
	$$q_T(\vA,\vx) \in \argmin_{\vw \in (0,\infty)^{|T|}} \|\vx-\vA_{:,T}\vw\|.$$
	Thus, since $(0,\infty)^{|T|}$ is an open set, we have that $q_T(\vA,\vx)$ is a local minimum of the function
	$$\vw \mapsto \|\vx- \vA_{:,T} \vw\|.$$
	Since $\vA$ has full column rank, $\vA_{:,T}$, a matrix whose columns are a subset of those of $\vA$, must also have full column rank. But then we know that this function has only one local minimum, equal to its global minimum, and is given by $\vA_{:,T}^\dagger  \vx$. Thus, we have that
	$$q_T(\vA,\vx) = \vA_{:,T}^\dagger  \vx.$$

	Using the fact that, by the definition of support, $q_{T^c}(\vA,\vx) = 0$, we arrive at the desired result.
\end{proof}

This result cannot be used to calculate $q(\vA,\vx)$ since it requires knowledge \textit{a priori} of the answer's support. However, if the support of $q(\vA,\vx)$ is locally constant in $(\vA,\vx)$-space, then once we have calculated the support of $q(\vA,\vx)$ we can use this formula to differentiate $q$ with respect to $\vA$ and $\vx$. Lemma \ref{measure_zero} guarantees that we can do this in almost all circumstances.  The following lemma will be used in the proof of Lemma \ref{measure_zero}.

\begin{lemma} \label{q_cts}
	Let $V_{r,s}$ be the set of matrices in $\mathbb{R}_+^{r \times s}$ ($r \geq s$) with full column rank. Then $q$ is continuous on $V_{r,s} \times \mathbb{R}_+^n$.
\end{lemma}

\begin{remark}
In the following proof, all matrix norms are taken to be spectral norms, rather than Frobenius norms.
\end{remark}

\begin{proof}
Fix a point $(\vA, \vx)$ in $V_{r,s} \times \mathbb{R}_+^r$; we will show $q$ is continuous there. Let $N$ be a bounded neighborhood of $\vA$ such that $\overline{N} \subset V_{r,s}$. Let $p_T : \mathbb{R}^{r \times s} \rightarrow \mathbb{R}^{r \times |T|}$ be the projection which maps $\vA \mapsto \vA_{:,T}$ for $T \subset [s]$. As a projection map, $p_T$ is clearly continuous, and so $p_T \left(\overline{N} \right)$ is compact, since $\overline{N}$ is compact. Since $\vA'_{:,T}$ must have full column rank if $\vA'$ does, and since $\overline{N} \subset V_{r,s}$, we must have that $p_T \left(\overline{N} \right) \subset V_{r,|T|}$. Since the pseudoinverse operation is continuous on sets of matrices with constant rank, it is continuous on $p_T \left(\overline{N} \right)$. Additionally, the spectral norm is continuous everywhere, so we get that the map $n_T : \overline{N} \rightarrow \mathbb{R}$ given by
$$\vA' \mapsto \left\|{\vA'_{:,T}}^\dagger \right\|$$
is continuous for each nonempty $T \subset [s]$. By the extreme value theorem, for each nonempty $T \subset [s]$, there is a matrix $\vA^T$ in $\overline{N}$ such that
$$\left\|{\vA^T_{:,T}}^\dagger \right \| \geq \left\|{\vA'_{:,T}}^\dagger \right\|$$
for all $\vA'$ in $\overline{N}$. Then we let
$$J_1 = \max_{\substack{T \subset [s] \\ T \neq \emptyset}} \left\|{\vA^T_{:,T}}^\dagger \right\|.$$

Thus, for all $\vA'$ in $\overline{N}$, we must have that $\left\|{\vA'_{:,T}}^\dagger \right\| \leq J_1$ for any set of indices $T \subset [s]$. In particular, for any $(\vA',\vx')$ in $N \times \mathbb{R}_+^r$, by Lemma \ref{formula_for_q}, we have that for $T = \text{supp}(q(\vA',\vx'))$, $q_T(\vA',\vx') = {\vA'_{:,T}}^\dagger \vx'$ and $q_{T^c}(\vA',\vx') = 0$, so
$$\|q(\vA',\vx')\| = \|q_T(\vA',\vx')\| = \left \|{\vA'_{:,T}}^\dagger \vx' \right\| \leq \left\|{\vA'_{:,T}}^\dagger \right\| \| \vx'\| \leq J_1 \|\vx'\|,$$
where the last inequality follows from the definition of $J_1$.
Let $N'$ be a bounded neighborhood of $\vx$ in $\mathbb{R}_+^r$. Let $J_2 = \sup_{x' \in N'} \|x'\|$. This finally gives that for all $(\vA',\vx') $ in $N \times N'$,
$$\|q(\vA',\vx')\| \leq J_1 \|\vx'\| \leq J_1 J_2.$$

Now, suppose $(\vA', \vx')$ is in $N \times N'$. By the definition of $q$, we have that
\begin{equation} \label{axmin}
\|\vx - \vA q(\vA,\vx)\| \leq \|\vx - \vA q(\vA',\vx')\|,
\end{equation}
and
\begin{equation} \label{a'x'min}
\|\vx' - \vA' q(\vA',\vx')\| \leq \|\x' - \vA' q(\vA,\vx)\|.
\end{equation}
We claim that the map $\|\vx - \vA q(\cdot)\|$ is continuous at $(\vA,\vx)$, i.e., for each $\epsilon>0$, there is an open set $W \subset \mathbb{R}_+^{r \times s} \times \mathbb{R}_+^r$ containing $(\vA,\vx)$ such that for all $(\vA',\vx')$ in $W$ we have
$$\left|\|\vx - \vA q(\vA',\vx')\|-\|\vx - \vA q(\vA,\vx)\| \right| < \epsilon.$$
To see this, fix $\epsilon>0$. We have that for $(\vA', \x')$ and $(\vA'',\vx'')$ in $N \times N'$, by the reverse triangle inequality,
\begin{align*}
\left|\|\vx - \vA q(\vA'',\vx'')\| - \|\vx'-\vA'
q(\vA'',\vx'')\|\right| &\leq \left\|(\vx-\vA q(\vA'',\vx''))-(\vx'-\vA'q(\vA'',\vx'')) \right\| \\&= \left\|(\vx-\vx') - (\vA-\vA')q(\vA'',\vx'')\right\| \\&\leq \|\vx-\vx'\| + \|\vA-\vA'\|\|q(\vA'',\vx'')\| \\&\leq \|\vx-\vx'\| + J_1 J_2\|\vA-\vA'\|.
\end{align*}
In particular, the above holds for $(\vA'',\vx'') = (\vA,\vx)$ and $(\vA'',\vx'') = (\vA',\vx')$, so
$$\left|\|\vx - \vA q(\vA',\vx')\| - \|\vx'-\vA'q(\vA',\vx')\|\right| \leq\|\vx-\vx'\| + J_1 J_2\|\vA-\vA'\|,$$
$$\left|\|\vx - \vA q(\vA, \vx)\| - \|\vx'-\vA'q(\vA,\vx)\|\right| \leq\|\vx-\vx'\| + J_1 J_2\|\vA-\vA'\|.$$
From these, we get
\begin{equation} \label{a'x'bound}
\|\vx-\vA q(\vA',\vx')\| \leq \|\vx-\vx'\|+J_1 J_2\|\vA-\vA'\| + \|\vx'- \vA' q(\vA',\vx')\|,
\end{equation}
\begin{equation} \label{axbound}
-\|\vx-\vA q(\vA, \vx)\| \leq \|\vx-\vx'\| + J_1 J_2 \|\vA-\vA'\|-\|\vx'-\vA'q(\vA,\vx)\|.
\end{equation}
Now, using equation \eqref{axmin}, we get
$$\left|\|\vx - \vA q(\vA',\vx')\|-\|\vx - \vA q(\vA,\vx)\| \right| = \|\vx - \vA q(\vA',\vx')\| - \|\vx - \vA q(\vA,\vx)\|.$$
Now, using equations \eqref{a'x'bound} and \eqref{axbound}, we get
$$\|\vx - \vA q(\vA',\vx')\| - \|\vx - \vA q(\vA,\vx)\| \leq 2\left(\|\vx-\vx'\|+J_1 J_2\|\vA-\vA'\|\right) + \|\vx'-\vA' q(\vA',\vx')\| -\|\vx'-\vA'q(\vA,\vx)\|.$$
Finally, applying equation \eqref{a'x'min}, we have
$$\|\vx'-\vA' q(\vA',\vx')\| -\|\vx'-\vA'q(\vA,\vx)\| \leq 0.$$
Putting this all together gives
$$\left|\|\vx - \vA q(\vA',\vx')\|-\|\vx - \vA q(\vA,\vx)\| \right| \leq 2\left(\|\vx-\vx'\|+J_1 J_2\|\vA-\vA'\|\right),$$
for any $(\vA',\vx')$ in $N \times N'$.
Let
$$\delta = \frac{\epsilon}{8 \max(J_1 J_2,1)}.$$
Then let $B_\delta (\vx)$ be the open $\ell_2$ ball centered at $\vx$ in $\mathbb{R}^r$ of radius $\delta$, and $B_\delta (\vA)$ be the open spectral norm ball centered at $\vA$ in $\mathbb{R}^{r \times s}$ of radius $\delta$. Then we have that for all $(\vA', \vx')$ in $(B_\delta(\vx) \cap N) \times (B_\delta(\vA) \cap N')$,
$$\left|\|\vx - \vA q(\vA',\vx')\|-\|\vx - \vA q(\vA,\vx)\| \right| \leq 2\left(\|\vx-\vx'\|+J_1 J_2\|\vA-\vA'\|\right) \leq 2(\delta + J_1 J_2 \delta) < \epsilon.$$
Thus, setting $W = (B_\delta(\vx) \cap N) \times (B_\delta(\vA) \cap N')$ gives what we wanted.

Now we show that $q$ is continuous at $(\vA,\vx)$. We fix $\epsilon>0$ and attempt to find an open set $W$ of $(\vA,\vx)$ so that $\|q(\vA,\vx)-q(\vA',\vx')\| < \epsilon$ for all $(\vA',\vx')$ in $W$. Since $q(\vA,\vx)$ is the unique global minimum (since $\vA$ has full column rank) of the convex function
$$\vs \mapsto \|\vx-\vA \vs\|$$
on $\mathbb{R}_+^r$, we can apply \cite[Theorem 27.2]{R70}. Thus, there is a $\delta>0$ such that, if $\vs$ is a nonnegative vector, then
\begin{equation} \label{f_inv_near_min}
\|\vx-\vA \vs\| < \|\vx-\vA q(\vA,\vx)\| + \delta \implies \|\vs-q(\vA,\vx)\| < \epsilon.
\end{equation}
Then since we proved that the function $\|\vx-\vA q(\cdot)\|$ is continuous at $(\vA,\vx)$, there is a neighborhood $W$ of $(\vA,\vx)$ in $\mathbb{R}_+^{r \times s}\times \mathbb{R}_+^r$ such that for all $(\vA',\vx')$ in $W$,
$$\left|\|\vx-\vA q(\vA',\vx')\|^2 - \|\vx-\vA q(\vA,\vx)\|^2\right|<\delta.$$
Then clearly for all $(\vA',\vx')$ in $W$,
$$\|\vx-\vA q(\vA',\vx')\|^2 < \|\vx-\vA q(\vA,\vx)\|^2+\delta,$$
so by \eqref{f_inv_near_min}, we get
$$\|q(\vA',\vx')-q(\vA,\vx)\|<\epsilon$$
for all $(\vA',\vx')$ in $W$, making $W$ our desired neighborhood. Thus, $q$ is continuous at $(\vA,\vx)$. Since $(\vA,\vx)$ was chosen arbitrarily in $V_{r,s} \times \mathbb{R}_+^r$, we immediately have that $q$ is continuous on this entire set, as desired.

\end{proof}

We may now prove Lemma \ref{measure_zero}.

\begin{proof}[Proof of Lemma \ref{measure_zero}]
Since in the regular Lebesgue measure, we have that full-measure sets are dense, it suffices to show that $U_{r,s}$ is open and has full-measure. First we show that it is open. We again adopt the notation $V_{r,s}$ to denote the set of matrices in $\mathbb{R}_+^{r \times s}$ ($r \geq s$) with full (column) rank. By its definition, since $q$ is only defined on $V_{r,s} \times \mathbb{R}_+^n$, we have that $(\vA,\vx) \in U_{r,s}$ exactly when there is a neighborhood $N$ of $(\vA,\vx)$ contained in $V_{r,s} \times \mathbb{R}_+^r$ such that $\supp q(\vA',\vx')$ is constant on $N$. For each $(\vA,\vx) \in U_{r,s}$, let $N_{\vA,\vx}$ be such a neighborhood. Then for all $(\vA,\vx)$ in $U_{r,s}$, $N_{\vA,\vx} \subset U_{r,s}$, since for each $(\vA',\vx')$ in $N_{\vA,\vx}$, $N_{\vA,\vx}$ is a neighborhood of $(\vA',\vx')$ in $V_{r,s} \times \mathbb{R}_+^r$ on which the support of $q$ is constant, so $(\vA',\vx') \in U_{r,s}$. Thus, we must have
$$U_{r,s} = \bigcup_{(\vA,\vx) \in U_{r,s}}N_{\vA,\vx},$$
which demonstrates that $U_{r,s}$ is open.

We now show that $U^c_{r,s}$ has zero measure. By the definition of $U_{r,s}$, we have
$$U^c_{r,s} = V_{r,s}^c \times \mathbb{R}_+^r \cup \left\{(\vA,\vx) \in V_{r,s} \times \mathbb{R}_+^r \; \Big| \; \exists (\vA^k,\vx^k) \rightarrow (\vA,\vx) \text{ s.t. } \supp q(\vA^k,\vx^k) \neq q(\vA,\vx)\right\}.$$
Since $V_{r,s}^c$ has zero measure, so does $V_{r,s}^c \times \mathbb{R}_+^r$, meaning that it suffices to show that
$$W =\left\{(\vA,\vx) \in V_{r,s} \times \mathbb{R}_+^r \; \Big| \; \exists (\vA^k,\vx^k) \rightarrow (\vA,\vx) \text{ s.t. } \supp q(\vA^k,\vx^k) \neq q(\vA,\vx)\right\}$$
has zero measure. Suppose that $(\vA,\vx)$ is in $W$. Then there is a sequence $(\vA^k,\vx^k)$ in $V_{r,s} \times \mathbb{R}_+^r$ such that $(\vA^k,\vx^k) \rightarrow (\vA,\vx)$ but $\supp q(\vA^k,\vx^k) \neq \supp q(\vA,\vx)$. Since there are only finitely many possible supports, at least one support must be represented infinitely often in the sequence $\supp q(\vA^k,\vx^k)$, so, by possibly restricting to a subsequence, we may assume without loss of generality that $\supp q(\vA^k,\vx^k)$ is constant. Let $T_1 = \supp(\vA^k,\vx^k)$ and let $T_0 = \supp (\vA,\vx)$. By hypothesis, $T_0 \neq T_1$.

We claim that $T_0 \subset T_1$. To see this, suppose it were not true. Then there would be an index $i$ in $T_0$ which is not in $T_1$. Then, since $i$ is in $T_0$, $q_i(\vA,\vx)>0$. On the other hand, $i$ is not in $T_1$, so $q_i(\vA^k,\vx^k)=0$. Since $(\vA^k,\vx^k) \rightarrow (\vA,\vx)$, and $q$ is continuous by Lemma \ref{q_cts}, we must have that $q_i(\vA^k,\vx^k) \rightarrow q_i(\vA,\vx)$, but this would imply that the zero sequence converges to a positive number, which is false. This contradiction shows that $T_0 \subset T_1$.

Since $T_0 \subset T_1$ and $T_0 \neq T_1$, there must be an index $i$ in $T_1$ which is not in $T_0$. Let $T_1(i)$ be the index of $i$ in $T_1$, i.e., so that $\vF_{:,i}$ and $\left(\vF_{:,T_1}\right)_{:,T_1(i)}$ are the same vector. Then we have that, by Lemma \ref{formula_for_q},
$$q_i(\vA^k,\vx^k) = \left({\vA^k_{:,T_1}}^\dagger \vx^k \right)_{T_1(i)},$$
whereas $q_i(\vA^k,\vx^k) = 0$ since $i$ is not in $T_0$. Again, since $q$ is continuous, we must have that $q_i(\vA^k,\vx^k) \rightarrow q_i(\vA,\vx)$, so
$$\left({\vA^k_{:,T_1}}^\dagger \vx^k \right)_{T_1(i)} \rightarrow 0$$
as $k \rightarrow \infty$. Since $(\vA',\vx') \mapsto \left({\vA'_{:,T_1}}^\dagger x' \right)_{T_1(i)}$ is a continuous map on $V_{r,s} \times \mathbb{R}_+^r$ (as the pseudoinversion, projection, and matrix multiplication operations are all continuous on this set), and $(\vA^k,\vx^k) \rightarrow (\vA,\vx)$, we must have
$$\left({\vA^k_{:,T_1}}^\dagger x^k \right)_{T_1(i)} \rightarrow \left({\vA_{:,T_1}}^\dagger x \right)_{T_1(i)}.$$
Thus,
$$\left({\vA_{:,T_1}}^\dagger x \right)_{T_1(i)} = 0.$$

Since $(\vA,\vx)$ was chosen arbitrarily in $W$, we have shown that for every $(\vA,\vx)$ in $W$, there is a set of indices $T$ and an index $i$ in $T$ such that
$$\left({\vA_{:,T}}^\dagger \vx \right)_{T(i)} = 0.$$
Let $f_T: V_{r,s} \times \mathbb{R}_+^r \rightarrow \mathbb{R}^{|T|}$ be the map given by
\begin{equation} \label{ft}
(\vA,\vx) \mapsto {\vA_{:,T}}^\dagger \vx.
\end{equation}
Let $Z_d$ be the set of all vectors in $\mathbb{R}^d$ which have at least one zero entry. Then we have shown that
$$W \subset \bigcup_{T \subset [s]} f_T^{-1}(Z_{|T|}).$$
Since the union above is finite, it suffices to show that $f_T^{-1}(Z_{|T|})$ has zero measure for each $T \subset [s]$. So we are done if we can show that the preimage of measure zero sets under $f_T$ have measure zero, as $Z_d$ is a measure zero subset of $\mathbb{R}^d$ (it is just the union of the $k$ codimension-1 coordinate planes, each of which has measure zero).

By \cite[Theorem 2]{P87}, since $f_T$ is a map from a higher dimensional space to a lower dimensional space, this will be true if we can show that $f_T$ is continuous everywhere, differentiable almost everywhere, and has a full (row) rank derivative almost everywhere. Since pseudoinversion is continuous and differentiable on the full rank matrices, and column projection and matrix multiplication are differentiable everywhere, we immediately have that $f_T$ is continuous and differentiable everywhere on $V_{r,s} \times \mathbb{R}_+^r$. Thus, we are done if we can show that $f_T$ has a full rank derivative almost everywhere. We write the derivative of $f_T$ in block form,
$$df_T = \begin{pmatrix}\frac{\partial f_T}{\partial \vA} & \frac{\partial f_T}{\partial \vx} \end{pmatrix}, $$
where we have linearized the indices in $\vA$ (i.e., $\frac{\partial f_T}{\partial \vA}$ is a $|T| \times rs$ dimensional matrix). This will have full (row) rank if either $\frac{\partial f_T}{\partial \vA}$ or $\frac{\partial f_T}{\partial \vx}$ has full (row) rank. Looking at the formula \eqref{ft}, we can immediately see that
$$\frac{\partial f_T}{\partial \vx} = \vA_{:,T}^\dagger.$$
Since every $\vA$ in $V_{r,s}$ has full (column) rank, $\vA_{:,T}$ will also always have full (column) rank, so $\vA_{:,T}^\dagger$ always has full (row) rank. Thus, $\frac{\partial f_T}{\partial \vx}$ has full row-rank at every point in $V_{r,s} \times \mathbb{R}_+^r$, meaning $d f_T$ does as well, so we are done.

\end{proof}

We may now prove Theorem \ref{derivs_of_q}.

\begin{proof}[Proof of Theorem \ref{derivs_of_q}]
	Since $(\vA,\vx)$ is in $U_{r,s}$, there is a neighborhood $N$ of $(\vA, \vx)$ such that for every $(\vA', \vx')$ in $N$, $q(\vA', \vx')$ also has support $T$. Then for any $(\vA', \x')$ in $N$, by Lemma \ref{formula_for_q}, we have
	\begin{equation} \label{local_q}
	q_T(\vA', \vx') = {\vA'_{:,T}}^\dagger  \vx', \qquad \qquad q_{T^c}(\vA', \vx') = 0.
	\end{equation}
	From this, we can immediately see that differentiating $q$ with respect to $\vx'$ and plugging in $(\vA, \vx)$ gives equation \eqref{X_deriv_of_q}.

	On the other hand, by the second equation in \eqref{local_q}, if $\alpha \notin T$, then $q_{\alpha}(\vA', \vx') = 0$ for all $(\vA', \vx')$ in $N$, so
	$$\frac{\partial q_\alpha}{\partial \vA_{i,:}}(\vA', \vx') = 0$$
	for all $(\vA', \vx')$ in $N$, and therefore for $(\vA, \vx)$ in particular. Furthermore, if $\beta\notin T$, then neither ${\vA'_{:,T}}^\dagger  \vx'$ nor $0$ depends on $\vA'_{i,\beta}$, so
	$$\frac{\partial q}{\partial \vA_{i,\beta}}(\vA', \vx') = 0$$
	for all $(\vA', \vx') \in N$, and therefore for $(\vA, \vx)$ in particular. Thus, if either $\alpha \notin T$ or $\beta \notin T$, then
	$$\frac{\partial q_\alpha}{\partial \vA_{i,\beta}}(\vA, \vx) = 0.$$
	Equivalently, the above equation holds if $(\alpha,\beta) \in (T \times T)^c$. This gives the second equation in \eqref{A_deriv_of_q}.

	So all that is left is to show is that
	$$\left(\frac{\partial q}{\partial \vA_{i,:}}(\vA, \vx)\right)_{T,T} = -\left(\vA_{:,T}^\dagger \right)_{:,i}\left(\vA_{:,T}^\dagger  \vx\right)^\top + \left(\left(\vI-\vA_{:,T}\vA_{:,T}^\dagger \right)\vx\right)_{i}\vA_{:,T}^\dagger \left(\vA_{:,T}^\dagger \right)^\top.$$
	We have that
	$$\left(\frac{\partial q}{\partial \vA_{i,:}}(\vA', \vx')\right)_{T,T} = \frac{\partial q_T}{\partial \vA_{i,T}}(\vA', \vx') = \frac{\partial}{\partial \vA'_{i,T}}\left({\vA'_{:,T}}^\dagger  \vx' \right)$$
	for all $(\vA', \vx')$ in $N$ by the first equation \eqref{local_q}. In order to compute the last derivative above, we appeal to the formula for the derivative of the pseudoinverse operation. Specifically, if $B$ is a matrix, then we have that
	$$\frac{\partial (\vF^\dagger )}{\partial \vF_{i, \alpha}} = - \vF^\dagger  \vE^{i, \alpha} \vF^\dagger  + \vF^\dagger  {\vF^\dagger }^\top {\vE^{i, \alpha}}^\top(\vI-\vF\vF^\dagger ) + (\vI-\vF^\dagger  \vF){\vE^{i, \alpha}}^\top {\vF^\dagger }^\top \vF^\dagger,$$
	where $\vE^{i, \alpha}$ is a matrix of the same size as $\vF$, all of whose entries are $0$, except for the $(i,\alpha)^{th}$ entry, which is $1$. See Theorem 4.3 in \cite{GP73} for a proof of this formula. Eventually, we will plug in $\vA'_{:,T}$ in for $\vF$ in the formula above, and so since $\vA'$ has full column rank (as $(\vA, \vx) \in U_{r,s}$), so does $\vA'_{:,T}$, and so we may assume that $\vF$ has full column rank. In particular, this means that $\vI-\vF^\dagger \vF=0$, so the third term of the equation above drops out, giving
	$$\frac{\partial (\vF^\dagger )}{\partial \vF_{i, \alpha}} = - \vF^\dagger  \vE^{i, \alpha} \vF^\dagger  + \vF^\dagger  {\vF^\dagger }^\top {\vE^{i, \alpha}}^\top(\vI-\vF\vF^\dagger ).$$
	Note that, by our definition of $\vE^{i,\alpha}$, we have that for any matrices $\vF$ and $\vG$, we have
	$$(\vF' \vE^{i,\alpha}\vG')_{\beta,j} = \vF'_{\beta,i}\vG'_{\alpha,j},$$
	and, by taking transposes,
	$$(\vF'' {\vE^{i,\alpha}}^\top \vG'')_{\beta,j} = \vF''_{\beta,\alpha}\vG''_{i,j}.$$
	Thus,
	$$\frac{\partial (\vF^\dagger )_{\beta,j}}{\partial \vF_{i, \alpha}} = - (\vF^\dagger )_{\beta,i} (\vF^\dagger )_{\alpha,j} + (\vF^\dagger  {
		\vF^\dagger }^\top)_{\beta,\alpha} (\vI-\vF\vF^\dagger )_{i,j}.$$
	Then we get
	\begin{align*}
	\frac{\partial \left(\vF^\dagger \vx' \right)_\beta}{\partial \vF_{i,\alpha}} &= \frac{\partial \left((\vF^\dagger )_{\beta,:}\vx' \right)}{\partial \vF_{i,\alpha}} \\&= \frac{\partial (\vF^\dagger )_{\beta,:}}{\partial \vF_{i,\alpha}}\vx' \\&= \left(- (\vF^\dagger )_{\beta,i} (\vF^\dagger )_{\alpha,:} + (\vF^\dagger  {
		\vF^\dagger }^\top)_{\beta,\alpha} (\vI-\vF\vF^\dagger )_{i,:}\right)\vx' \\&= - (\vF^\dagger )_{\beta,i} (\vF^\dagger  \vx')_{\alpha} +\left((\vI-\vF\vF^\dagger ) \vx' \right)_i (\vF^\dagger  {\vF^\dagger }^\top)_{\beta,\alpha} .
	\end{align*}
	Now, the matrix $ (\vF^\dagger )_{:,i} (\vF^\dagger  \vx')^\top$ has $(\beta,\alpha)^{th}$ entry equal to $ (\vF^\dagger )_{\beta,i} (\vF^\dagger  \vx')_{\alpha}$, and so we get that
	$$\frac{\partial (\vF^\dagger  \vx')}{\partial \vF_{i,:}} = - (\vF^\dagger )_{:,i} (\vF^\dagger  \vx')^\top +\left((\vI-\vF\vF^\dagger ) \vx' \right)_i \vF^\dagger  {\vF^\dagger }^\top. $$
	Setting $\vF = \vA'_{:,T}$, this gives
	$$\left(\frac{\partial q}{\partial \vA_{i,:}}(\vA', \vx')\right)_{T,T} = \frac{\partial}{\partial \vA'_{i,T}}\left({\vA'_{:,T}}^\dagger  \vx' \right) = -\left({\vA'_{:,T}}^\dagger \right)_{:,i}\left({\vA'_{:,T}}^\dagger  \vx\right)^\top + \left(\left(\vI-\vA'_{:,T}{\vA'_{:,T}}^\dagger \right)\vx\right)_{i}{\vA'_{:,T}}^\dagger \left({\vA'_{:,T}}^\dagger \right)^\top,$$
	which holds for all $(\vA',\vx')$ in $N$. Plugging in $(\vA, \vx)$ for $(\vA', \vx')$ gives the desired result.

\end{proof}

We finally can prove Theorem \ref{full_backprop}, which gives the required derivatives to apply backpropagation.

\begin{proof}[Proof of Theorem \ref{full_backprop}]
	The chain rule tells us that
	$$\frac{\partial \loss}{\partial \vA^{(\ell_1)}} = \left(\frac{\partial \loss}{\partial \vA^{(\ell_1)}} \right)^{\vS} + \sum_{\substack{\ell_1 \leq \ell_2 \leq \cL \\ 1 \leq m \leq M \\ 1 \leq \alpha \leq k^{(\ell_2)}}}\left(\frac{\partial \loss}{\partial \vS^{(\ell_2)}_{\alpha,m}}\right)^\holdconstant \left( \frac{\partial \vS^{(\ell_2)}_{\alpha,m}}{\partial \vA^{(\ell_1)}}\right),$$
	since $\vS^{(\ell_2)}_{\alpha,m}$ for $\ell_1 \leq \ell_2 \leq \cL$ are the only variables which $C$ depends on, which themselves depend on $\vA^{(\ell_1)}$. Therefore we are done if we can show that
	$$\sum_{\alpha=1}^{k^{(\ell_2)}} \left(\frac{\partial \loss}{\partial \vS^{(\ell_2)}_{\alpha,m}}\right)^\holdconstant \left(  \frac{\partial \vS^{(\ell_2)}_{\alpha,m}}{\partial \vA^{(\ell_1)}}\right) = \vU^{(\ell_1,\ell_2),m}.$$

	Now, we apply Lemma \ref{q_acts_columnwise} to the second equation in \eqref{S_def} to get
	\begin{equation} \label{S_cols}
	\vS^{(\ell)}_{:,m} = q(\vA^{(\ell)},
	\vS^{(\ell-1)}_{:,m}).
	\end{equation}
	By repeated application of the chain rule to this equation, we get that
	\begin{equation} \label{big_product}
	\left(  \frac{\partial \vS^{(\ell_2)}_{:,m}}{\partial \vS^{(\ell_1)}_{:,m}}\right) = \left(  \frac{\partial \vS^{(\ell_2)}_{:,m}}{\partial \vS^{(\ell_2-1)}_{:,m}}\right)\left(  \frac{\partial \vS^{(\ell_2-1)}_{:,m}}{\partial \vS^{(\ell_1-2)}_{:,m}}\right) \dots \left(  \frac{\partial \vS^{(\ell_1+1)}_{:,m}}{\partial \vS^{(\ell_1)}_{:,m}}\right).
	\end{equation}
	Applying Theorem \ref{derivs_of_q} to equation \ref{S_cols} gives
	\begin{equation} \label{s_derivs}
	\left(\frac{\partial \vS^{(\ell)}_{:,m}}{\partial \vS^{(\ell-1)}_{:,m}} \right)_{T_m^{(\ell)},:} = {\vA^{(\ell)}_{:,T_{m}^{(\ell)}}}^{\dagger}, \qquad \left(\frac{\partial \vS^{(\ell)}_{:,m}}{\partial \vS^{(\ell-1)}_{:,m}} \right)_{{T_m^{(\ell)}}^c,:} = 0.
	\end{equation}
	Now we use the following fact: if $\vF$ and $\vG$ are compatible matrices, $R$ is a subset of the row indices of $\vG$, and $\vG_{R^c,:} = 0$, then
	\begin{equation} \label{fact}
	\vF\vG = \vF_{:,R} \vG_{R,:}.
	\end{equation}
	Combining equation \eqref{big_product} and the second equation in \eqref{s_derivs} with this fact gives
	$$\left(  \frac{\partial \vS^{(\ell_2)}_{:,m}}{\partial \vS^{(\ell_1)}_{:,m}}\right) = \left(  \frac{\partial \vS^{(\ell_2)}_{:,m}}{\partial \vS^{(\ell_2-1)}_{:,m}}\right)_{:,T_m^{(\ell_2-1)}}\left(  \frac{\partial \vS^{(\ell_2-1)}_{:,m}}{\partial \vS^{(\ell_1-2)}_{:,m}}\right)_{T_m^{(\ell_2-1)},T_m^{(\ell_2-2)}} \dots \left(  \frac{\partial \vS^{(\ell_1+1)}_{:,m}}{\partial \vS^{(\ell_1)}_{:,m}}\right)_{T_m^{(\ell_1+1)},:}.$$
	And now we can apply the first equation in \eqref{s_derivs} to get
	\begin{equation} \label{dsds}
	\left(\frac{\partial \vS^{(\ell_2)}_{:,m}}{\partial \vS^{(\ell_1)}_{:,m}}\right)_{T_m^{(\ell_2)},:} = \left({\vA^{(\ell_2)}_{:,T_m^{(\ell_2)}}}^{\dagger}\right)_{:,T_m^{(\ell_2-1)}}\left({\vA^{(\ell_2-1)}_{:,T_m^{(\ell_2-1)}}}^{\dagger}\right)_{:,T_m^{(\ell_2-2)}}\dots \left({\vA^{(\ell_1+1)}_{:,T_m^{(\ell_1+1)}}}^{\dagger}\right) = \vPhi^{(\ell_1+1,\ell_2),m}, \qquad \left(\frac{\partial \vS^{(\ell_2)}_{:,m}}{\partial \vS^{(\ell_1)}_{:,m}}\right)_{{T_m^{(\ell_2)}}^c,:} = 0.
	\end{equation}
	Now, we apply the chain rule again to \ref{S_cols}, giving
	$$\left(  \frac{\partial \vS^{(\ell_2)}_{:,m}}{\partial \vA^{(\ell_1)}_{i,:}}\right) = \left(  \frac{\partial \vS^{(\ell_2)}_{:,m}}{\partial \vS^{(\ell_1)}_{:,m}}\right)\left(  \frac{\partial \vS^{(\ell_1)}_{:,m}}{\partial \vA^{(\ell_1)}_{i,:}}\right).$$
	On the other hand,
	$$\sum_{\alpha=1}^{k^{(\ell_2)}} \left(\frac{\partial \loss}{\partial \vS^{(\ell_2)}_{\alpha,m}}\right)^\holdconstant \left(  \frac{\partial \vS^{(\ell_2)}_{\alpha,m}}{\partial \vA^{(\ell_1)}_{i,:}}\right) =  \left(\left(\frac{\partial \loss}{\partial \vS^{(\ell_2)}_{:,m}}\right)^{\holdconstant}\right)^{\top} \left(  \frac{\partial \vS^{(\ell_2)}_{:,m}}{\partial \vA^{(\ell_1)}_{i,:}}\right),$$
	where $\left(\frac{\partial \loss}{\partial \vS^{(\ell_2)}_{:,m}}\right)$ is taken to be a column vector. Thus,
	$$\sum_{\alpha=1}^{k_{\ell_2}} \left(\frac{\partial \loss}{\partial \vS^{(\ell_2)}_{\alpha,m}}\right)^\holdconstant \left(  \frac{\partial \vS^{(\ell_2)}_{\alpha,m}}{\partial \vA^{(\ell_1)}_{i,:}}\right) =\left(\left(\frac{\partial \loss}{\partial \vS^{(\ell_2)}_{:,m}}\right)^{\holdconstant}\right)^{\top} \left(  \frac{\partial \vS^{(\ell_2)}_{:,m}}{\partial \vS^{(\ell_1)}_{:,m}}\right)\left(  \frac{\partial \vS^{(\ell_1)}_{:,m}}{\partial \vA^{(\ell_1)}_{i,:}}\right).$$
	We can apply equation \eqref{fact} again, using both parts of \ref{dsds}, giving
	$$\sum_{\alpha=1}^{k^{(\ell_2)}} \left(\frac{\partial \loss}{\partial \vS^{(\ell_2)}_{\alpha,m}}\right)^\holdconstant \left(  \frac{\partial \vS^{(\ell_2)}_{\alpha,m}}{\partial \vA^{(\ell_1)}_{i,:}}\right) = \left(\left(\left(\frac{\partial \loss}{\partial \vS^{(\ell_2)}_{:,m}}\right)^{\holdconstant}\right)^{\top} \right)_{:,T_m^{(\ell_2)}} \left(  \frac{\partial \vS^{(\ell_2)}_{:,m}}{\partial \vS^{(\ell_1)}_{:,m}}\right)_{T_m^{(\ell_2)},:}\left(  \frac{\partial \vS^{(\ell_1)}_{:,m}}{\partial \vA^{(\ell_1)}_{i,:}}\right) = \left(\left(\frac{\partial \loss}{\partial \vS^{(\ell_2)}}\right)^\holdconstant_{T_m^{(\ell_2)},m} \right)^\top \vPhi^{(\ell_1+1,\ell_2),m}\left(  \frac{\partial \vS^{(\ell_1)}_{:,m}}{\partial \vA^{(\ell_1)}_{i,:}}\right).$$
	By Theorem \ref{derivs_of_q}, the rows of $\left(  \frac{\partial \vS^{(\ell_1)}_{:,m}}{\partial \vA^{(\ell_1)}_{i,:}}\right)$ with indices outside $T_{m}^{(\ell_1)}$, are zero, so we can apply our fact again to get
	$$\sum_{\alpha=1}^{k^{(\ell_2)}} \left(\frac{\partial \loss}{\partial \vS^{(\ell_2)}_{\alpha,m}}\right)^\holdconstant \left(  \frac{\partial \vS^{(\ell_2)}_{\alpha,m}}{\partial \vA^{(\ell_1)}_{i,:}}\right) = \left(\left(\frac{\partial \loss}{\partial \vS^{(\ell_2)}}\right)^\holdconstant_{T_m^{(\ell_2)},m} \right)^\top \vPhi^{(\ell_1+1,\ell_2),m}_{:,T_m^{(\ell_1)}}\left(  \frac{\partial \vS^{(\ell_1)}_{:,m}}{\partial \vA^{(\ell_1)}_{i,:}}\right)_{T_m^{(\ell_1)},:}.$$
	Now, Theorem \ref{derivs_of_q} again tells us that the columns of $\left(  \frac{\partial \vS^{(\ell_1)}_{:,m}}{\partial \vA^{(\ell_1)}_{r,:}}\right)$ outside $T_{m}^{(\ell_1)}$ are zero, so we have
	$$\left( \sum_{i=1}^{k_{\ell_2}} \left(\frac{\partial \loss}{\partial \vS^{(\ell_2)}_{i,m}}\right)^\holdconstant \left(  \frac{\partial \vS^{(\ell_2)}_{i,m}}{\partial \vA^{(\ell_1)}_{r,:}}\right)\right)_{:,{T_m^{(\ell_1)}}^c} = \left(\left(\frac{\partial \loss}{\partial \vS^{(\ell_2)}}\right)^\holdconstant_{T_m^{(\ell_2)},m} \right)^\top \vPhi^{(\ell_1+1,\ell_2),m}_{:,T_m^{(\ell_1)}}\left(  \frac{\partial \vS^{(\ell_1)}_{:,m}}{\partial \vA^{(\ell_1)}_{r,:}}\right)_{T_m^{(\ell_1)},{T_m^{(\ell_1)}}^c} = 0, $$
	which is what we wanted, since $\vU^{(\ell_1,\ell_2),m}_{:,{T_m^{(\ell_1)}}^c} = 0$. Thus, all that is left to show is that
	$$\left( \sum_{\alpha=1}^{k_{\ell_2}} \left(\frac{\partial \loss}{\partial \vS^{(\ell_2)}_{\alpha,m}}\right)^\holdconstant \left(  \frac{\partial \vS^{(\ell_2)}_{\alpha,m}}{\partial \vA^{(\ell_1)}_{i,:}}\right)\right)_{:,{T_m^{(\ell_1)}}} = \vU^{(\ell_1,\ell_2),m}_{i,T_m^{(\ell_1)}},$$
	i.e.,
	\begin{equation} \label{to_show}
	\left( \sum_{\alpha=1}^{k^{(\ell_2)}} \left(\frac{\partial \loss}{\partial \vS^{(\ell_2)}_{\alpha,m}}\right)^\holdconstant \left(  \frac{\partial \vS^{(\ell_2)}_{\alpha,m}}{\partial \vA^{(\ell_1)}_{i,:}}\right)\right)_{:,{T_m^{(\ell_1)}}} = - \vd^{(\ell_1,\ell_2),m}_{i}\left({\vS^{(\ell_1)}_{T_m^{(\ell_1)},m}}\right)^\top + \left(\vS^{(\ell_1 -1)}-\vA^{(\ell_1)}\vS^{(\ell_1)}\right)_{i,m} \left(\vd^{(\ell_1,\ell_2,),m}\right)^\top \left({\vA^{(\ell_1)}_{:,T_m^{(\ell_1)}}}^\dagger\right)^\top.
	\end{equation}
	Applying Theorem \ref{derivs_of_q} once more gives that
	\begin{align}
	\nonumber
	&\left( \sum_{\alpha=1}^{k^{(\ell_2)}} \left(\frac{\partial \loss}{\partial \vS^{(\ell_2)}_{\alpha,m}}\right)^\holdconstant \left(  \frac{\partial \vS^{(\ell_2)}_{\alpha,m}}{\partial \vA^{(\ell_1)}_{i,:}}\right)\right)_{:,{T_m^{(\ell_1)}}} \\ \nonumber &= \left(\left(\frac{\partial \loss}{\partial \vS^{(\ell_2)}}\right)^\holdconstant_{T_m^{(\ell_2)},m} \right)^\top \vPhi^{(\ell_1+1,\ell_2),m}_{:,T_m^{(\ell_1)}}\left(  \frac{\partial \vS^{(\ell_1)}_{:,m}}{\partial \vA^{(\ell_1)}_{i,:}}\right)_{T_m^{(\ell_1)},{T_m^{(\ell_1)}}} \\  &= \left(\left(\frac{\partial \loss}{\partial \vS^{(\ell_2)}}\right)^\holdconstant_{T_m^{(\ell_2)},m} \right)^\top \vPhi^{(\ell_1+1,\ell_2),m}_{:,T_m^{(\ell_1)}} \Bigg( -\left({\vA^{(\ell_1)}_{:,T_m^{(\ell_1)}}}^\dagger\right)_{:,i}\left({\vA^{(\ell_1)}_{:,T_m^{(\ell_1)}}}^\dagger \vS^{(\ell_1-1)}_{:,m}\right)^\top \\&\label{current} \hphantom{\left(\left(\frac{\partial \loss}{\partial \vS^{(\ell_2)}}\right)_{T_m^{(\ell_2)},m} \right)^\top \vPhi^{(\ell_1+1,\ell_2),m}_{:,T_m^{(\ell_1)}} + \Bigg (  }+ \left(\left(\vI-\vA^{(\ell_1)}_{:,T_m^{(\ell_1)}} {\vA^{(\ell_1)}_{:,T_m^{(\ell_1)}}}^\dagger\right)\vS^{(\ell_1-1)}_{:,m}\right)_{i}{\vA^{(\ell_1)}_{:,T_m^{(\ell_1)}}}^\dagger\left({\vA^{(\ell_1)}_{:,T_m^{(\ell_1)}}}^\dagger\right)^\top\Bigg).
	\end{align}
	Now, using equation \eqref{S_cols} and Lemma \ref{formula_for_q}, we get that
	\begin{equation} \label{incr_s}
	{\vA^{(\ell_1)}_{:,T_m^{(\ell_1)}}}^\dagger \vS^{(\ell_1-1)}_{:,m} = \vS^{(\ell_1)}_{T_m^{(\ell_1)},m}.
	\end{equation}
	Thus, we have,
	\begin{align} \nonumber
	\left(\left(\frac{\partial \loss}{\partial \vS^{(\ell_2)}}\right)^\holdconstant_{T_m^{(\ell_2)},m} \right)^\top \vPhi^{(\ell_1+1,\ell_2),m}_{:,T_m^{(\ell_1)}} &\left( -\left({\vA^{(\ell_1)}_{:,T_m^{(\ell_1)}}}^\dagger\right)_{:,i}\left({\vA^{(\ell_1)}_{:,T_m^{(\ell_1)}}}^\dagger \vS^{(\ell_1-1)}_{:,m}\right)^\top \right) \\ \label{first_term} &= - \left(\left(\left(\frac{\partial \loss}{\partial \vS^{(\ell_2)}}\right)^\holdconstant_{T_m^{(\ell_2)},m} \right)^\top \vPhi^{(\ell_1+1,\ell_2),m}_{:,T_m^{(\ell_1)}} {\vA^{(\ell_1)}_{:,T_m^{(\ell_1)}}}^\dagger \right)_{:,i} \left( \vS^{(\ell_1)}_{T_m^{(\ell_1)},m}\right)^\top.
	\end{align}
	From the definition of $\vPhi$, we see that
	\begin{equation} \label{phi_recursive}
	\vPhi^{(\ell_1+1,\ell_2),m}_{:,T_m^{(\ell_1)}} {\vA^{(\ell_1)}_{:,T_m^{(\ell_1)}}}^\dagger  = \vPhi^{(\ell_1,\ell_2),m},
	\end{equation}
	so
	$$\left(\left(\frac{\partial \loss}{\partial \vS^{(\ell_2)}}\right)^\holdconstant_{T_m^{(\ell_2)},m} \right)^\top \vPhi^{(\ell_1+1,\ell_2),m}_{:,T_m^{(\ell_1)}} {\vA^{(\ell_1)}_{:,T_m^{(\ell_1)}}}^\dagger = \left(\vd^{(\ell_1,\ell_2),m} \right)^\top.$$
	Since $\vd^{(\ell_1,\ell_2),m}$ is a column vector, we get
	$$\left(\left(\left(\frac{\partial \loss}{\partial \vS^{(\ell_2)}}\right)^\holdconstant_{T_m^{(\ell_2)},m} \right)^\top \vPhi^{(\ell_1+1,\ell_2),m}_{:,T_m^{(\ell_1)}} {\vA^{(\ell_1)}_{:,T_m^{(\ell_1)}}}^\dagger\right)_{:,i} = \vd^{(\ell_1,\ell_2),m}_i.$$
	Plugging this into \ref{first_term}, we get
	$$\left(\left(\frac{\partial \loss}{\partial \vS^{(\ell_2)}}\right)^\holdconstant_{T_m^{(\ell_2)},m} \right)^\top \vPhi^{(\ell_1+1,\ell_2),m}_{:,T_m^{(\ell_1)}} \left( -\left({\vA^{(\ell_1)}_{:,T_m^{(\ell_1)}}}^\dagger\right)_{:,i}\left({\vA^{(\ell_1)}_{:,T_m^{(\ell_1)}}}^\dagger \vS^{(\ell_1-1)}_{:,m}\right)^\top \right) = -\vd^{(\ell_1,\ell_2),m}_i\left( \vS^{(\ell_1)}_{T_m^{(\ell_1)},m}\right)^\top.$$
	Thus, from equations \eqref{current} and \eqref{to_show}, it suffices to show that
	\begin{align*}
	\left(\left(\frac{\partial \loss}{\partial \vS^{(\ell_2)}}\right)^\holdconstant_{T_m^{(\ell_2)},m} \right)^\top \vPhi^{(\ell_1+1,\ell_2),m}_{:,T_m^{(\ell_1)}} &\left(\left(\vI-\vA^{(\ell_1)}_{:,T_m^{(\ell_1)}} {\vA^{(\ell_1)}_{:,T_m^{(\ell_1)}}}^\dagger\right)\vS^{(\ell_1-1)}_{:,m}\right)_{i}{\vA^{(\ell_1)}_{:,T_m^{(\ell_1)}}}^\dagger\left({\vA^{(\ell_1)}_{:,T_m^{(\ell_1)}}}^\dagger\right)^\top \\&= \left(\vS^{(\ell_1 -1)}-\vA^{(\ell_1)}\vS^{(\ell_1)}\right)_{i,m} \left(\vd^{(\ell_1,\ell_2,),m}\right)^\top \left({\vA^{(\ell_1)}_{:,T_m^{(\ell_1)}}}^\dagger\right)^\top.
	\end{align*}
	First, we note that, by equation \eqref{incr_s}, we have
	$$\left(\vI-\vA^{(\ell_1)}_{:,T_m^{(\ell_1)}} {\vA^{(\ell_1)}_{:,T_m^{(\ell_1)}}}^\dagger\right)\vS^{(\ell_1-1)}_{:,m} = \vS^{\ell_1-1}_{:,m} - \vA^{(\ell_1)}_{:,T_m^{(\ell_1)}}\vS^{(\ell_1)}_{T_m^{(\ell_1)},m}.$$
	Since $\vS^{(\ell_1)}_{{T_m^{(\ell_1)}}^c,m} = 0$, we apply equation \eqref{fact} once more to get that
	$$\vA^{(\ell_1)}_{:,T_m^{(\ell_1)}}\vS^{(\ell_1)}_{T_m^{(\ell_1)},m} = \vA^{(\ell_1)}\vS^{(\ell_1)}_{:,m}.$$
	Thus,
	$$\left(\left(\vI-\vA^{(\ell_1)}_{:,T_m^{(\ell_1)}} {\vA^{(\ell_1)}_{:,T_m^{(\ell_1)}}}^\dagger\right)\vS^{(\ell_1-1)}_{:,m}\right)_i = \left(\vS^{\ell_1-1}_{:,m} - \vA^{(\ell_1)}_{:,T_m^{(\ell_1)}}\vS^{(\ell_1)}_{T_m^{(\ell_1)},m}\right)_i =\left(\vS^{\ell_1-1}_{:,m} - \vA^{(\ell_1)}\vS^{(\ell_1)}_{:,m}\right)_i  =\left(\vS^{\ell_1-1}- \vA^{(\ell_1)}\vS^{(\ell_1)}\right)_{i,m}.$$
	This gives
	\begin{align*}
	\left(\left(\frac{\partial \loss}{\partial \vS^{(\ell_2)}}\right)^\holdconstant_{T_m^{(\ell_2)},m} \right)^\top \vPhi^{(\ell_1+1,\ell_2),m}_{:,T_m^{(\ell_1)}} &\left(\left(\vI-\vA^{(\ell_1)}_{:,T_m^{(\ell_1)}} {\vA^{(\ell_1)}_{:,T_m^{(\ell_1)}}}^\dagger\right)\vS^{(\ell_1-1)}_{:,m}\right)_{i}{\vA^{(\ell_1)}_{:,T_m^{(\ell_1)}}}^\dagger\left({\vA^{(\ell_1)}_{:,T_m^{(\ell_1)}}}^\dagger\right)^\top \\&= \left(\vS^{\ell_1-1}- \vA^{(\ell_1)}\vS^{(\ell_1)}\right)_{i,m}\left(\left(\frac{\partial \loss}{\partial \vS^{(\ell_2)}}\right)^\holdconstant_{T_m^{(\ell_2)},m} \right)^\top \vPhi^{(\ell_1+1,\ell_2),m}_{:,T_m^{(\ell_1)}} {\vA^{(\ell_1)}_{:,T_m^{(\ell_1)}}}^\dagger\left({\vA^{(\ell_1)}_{:,T_m^{(\ell_1)}}}^\dagger\right)^\top.
	\end{align*}
	Using equation \eqref{phi_recursive} and the definition of $\vd$, we get that
	$$\left(\left(\frac{\partial \loss}{\partial \vS^{(\ell_2)}}\right)^\holdconstant_{T_m^{(\ell_2)},m} \right)^\top \vPhi^{(\ell_1+1,\ell_2),m}_{:,T_m^{(\ell_1)}} {\vA^{(\ell_1)}_{:,T_m^{(\ell_1)}}}^\dagger = \left(\left(\frac{\partial \loss}{\partial \vS^{(\ell_2)}}\right)^\holdconstant_{T_m^{(\ell_2)},m} \right)^\top \Phi^{(\ell_1,\ell_2),m} = \left(\vd^{(\ell_1,\ell_2),m} \right)^\top.$$
	Thus,
	\begin{align*}
	\left(\left(\frac{\partial \loss}{\partial \vS^{(\ell_2)}}\right)^\holdconstant_{T_m^{(\ell_2)},m} \right)^\top \vPhi^{(\ell_1+1,\ell_2),m}_{:,T_m^{(\ell_1)}} &\left(\left(\vI-\vA^{(\ell_1)}_{:,T_m^{(\ell_1)}} {\vA^{(\ell_1)}_{:,T_m^{(\ell_1)}}}^\dagger\right)\vS^{(\ell_1-1)}_{:,m}\right)_{i}{\vA^{(\ell_1)}_{:,T_m^{(\ell_1)}}}^\dagger\left({\vA^{(\ell_1)}_{:,T_m^{(\ell_1)}}}^\dagger\right)^\top \\&= \left(\vS^{\ell_1-1}- \vA^{(\ell_1)}\vS^{(\ell_1)}\right)_{i,m}\left(\vd^{(\ell_1,\ell_2),m} \right)^\top\left({\vA^{(\ell_1)}_{:,T_m^{(\ell_1)}}}^\dagger\right)^\top,
	\end{align*}
	which is what we wanted.

\end{proof}

\end{document}